\documentclass[journal]{IEEEtran}
\ifCLASSINFOpdf
\else
\fi
\usepackage{amssymb}
\usepackage{amsthm}
\usepackage{amsmath}
\usepackage{mathrsfs}
\usepackage{latexsym,bm}
\usepackage{graphicx,subfigure,float}
\usepackage{booktabs,multirow}
\usepackage{color}
\usepackage{footnote,verbatim}
\usepackage{algorithm}
\usepackage{algpseudocode}
\newtheorem{theorem}{Theorem}
\newtheorem{lemma}{Lemma}

\begin{document}
%
\title{Robust Locality-Aware Regression \\for Labeled Data Classification}
%
%
%

\author{Liangchen~Hu and~Wensheng~Zhang
\thanks{This work was supported in part by the National Key R$\&$D Program of China under Grant 2017YFC0803700, and in part by the National Natural Science Foundation of China under Grant U1636220 and Grant 61876183. \it{(Corresponding author: Wensheng Zhang.)}}
\thanks{L. Hu is with the School of Computer Science and Engineering, Nanjing University of Science and Technology, Nanjing 210094, China (e-mail: hlc\_clear@foxmail.com).}
\thanks{W. Zhang is with the Research Center of Precision Sensing and Control, Institute of Automation, Chinese Academy of Sciences, Beijing, 100190, China, and also with the University of Chinese Academy of Sciences, Beijing, 101408, China (e-mail: zhangwenshengia@hotmail.com).}
\thanks{Manuscript received April 19, 2005; revised August 26, 2015.}
}
\maketitle

\begin{abstract}
With the dramatic increase of dimensions in the data representation, extracting latent low-dimensional features becomes of the utmost importance for efficient classification. Aiming at the problems of unclear margin representation and difficulty in revealing the data manifold structure in most of the existing linear discriminant methods, we propose a new discriminant feature extraction framework, namely Robust Locality-Aware Regression (RLAR). In our model, we introduce a retargeted regression to perform the marginal representation learning adaptively instead of using the general average inter-class margin. Besides, we formulate a new strategy for enhancing the local intra-class compactness of the data manifold, which can achieve the joint learning of locality-aware graph structure and desirable projection matrix. To alleviate the disturbance of outliers and prevent overfitting, we measure the regression term and locality-aware term together with the regularization term by the $L_{2,1}$ norm. Further, forcing the row sparsity on the projection matrix through the $L_{2,1}$ norm achieves the cooperation of feature selection and feature extraction. Then, we derive an effective iterative algorithm for solving the proposed model. The experimental results over a range of UCI data sets and other benchmark databases demonstrate that the proposed RLAR outperforms some state-of-the-art approaches.
\end{abstract}


%
\IEEEpeerreviewmaketitle

\section{Introduction}
%
%
%
%
\IEEEPARstart{E}{xcessive} dimensionality leads to high storage overhead, heavy computation and huge time consumption in the training process of machine learning. And, as the ambient space expands exponentially with the increase of dimensionality, the proportion of training data in the whole data space drops sharply, thus resulting in the worse generalization of the training model \cite{Ladla2011,Ma2018}. A significant way to address these issues is dimensionality reduction (DR) \cite{Passalis2018,Pang2019}, which transforms the original high-dimensional spatial data into a low-dimensional subspace by some resultful means. Feature selection (FS) and feature extraction (FE) are two main techniques for processing the DR problems of high-dimensional data \cite{Khalid2014}. FS aims at learning a way for eliminating redundant features in the original space, while FE gains compact features with stronger recognition through recombination of original features in the process of spatial transformation. By contrast, FE is a more powerful means.

Among the many FE methods, principal component analysis (PCA) \cite{Hotelling1933} and linear discriminant analysis (LDA) \cite{Fisher1936} are two commonly-used unsupervised and supervised techniques respectively. PCA maximizes the divergence of the mapped data without considering label information during training, while LDA makes the mapped samples have better classification performance based on discriminability. Unfortunately, the mean dependence of LDA makes it incapable of revealing the data manifold structure, thus resulting in poor classification performance. Otherwise, multiple clusters usually happens to be formed in the same class \cite{Sugiyama2007}, such as odd-even classification of handwritten digits, multiple triggers of a single disease, etc., such data conforms to non-Gaussian distribution \cite{Boutemedjet2009,Luo2018,Li2017}, which challenges LDA of Gaussian distribution hypothesis. Although improved for diversification purposes, many variants of LDA still inherit this defect due to the problem of unchanged mean dependence, for example, orthogonal LDA (OLDA) \cite{Ye2006}, maximum margin criterion (MMC) \cite{Li2006}, sparse uncorrelated LDA (SULDA) \cite{Zhang2016}, robust LDA (RLDA) \cite{Zhao2019}, robust sparse LDA (RSLDA) \cite{Wen2019}, etc. To appropriately reduce the dimensionality of data and improve the computational efficiency while maintaining high classification performance, preserving the local manifold structure is crucial to success. Inspired by this, Sugiyama improved LDA's scatter loss into the form of sample pairs by combining the idea of locality preserving projections (LPPs) \cite{He2004}, namely local Fisher discriminant analysis (LFDA) \cite{Sugiyama2007}. Cai et al. \cite{Cai2007} proposed locality sensitive discriminant analysis (LSDA), which aims to mine the potential subspace in the way of perceiving the local geometric structure of data manifold, in which the nearby samples having the same label are close to each other instead of being far away from each other. Similarly, Nie et al. \cite{Nie2007} proposed neighborhood minmax projections (NMMPs) formulated by pairwise samples and derived an effective strategy for solving trace ratio optimization problem. Besides, Fan et al. \cite{Fan2011} presented an improved LDA framework, the local LDA (LLDA), which can perform well without satisfying the assumption of Gaussian distribution.

Considering the aforementioned methods in a unified graph embedding framework, they differ in the way of constructing the graphs within and between classes, including the connection and affinity of the sample pairs. In the view of data manifold recovery \cite{Seung2000,Belkin2002}, the sample relationships in the intra-class graph constructed in the original space are expected to be preserved completely in the low-dimensional embedding space.
However, the linear projections used in these methods make nonlinear manifold recovery almost impossible. In other words, the intra-class graph constructed in the original space is not optimal. Therefore, some new strategies have emerged to try to learn the optimal intra-class graphs while learning the optimal projections. On the premise that the affinity of intra-class samples satisfies the quadratic condition, Li et al. \cite{Li2017} studied the full-connection intra-class graph, and proposed locality adaptive discriminant analysis (LADA) which can well reveal the potential data manifold structure. Luo et al. \cite{Luo2018} proposed another adaptive discriminant analysis (ADA), which formulates the intra-class affinity loss in the form of heat kernel function and approximates it by quadratic model. Instead of investigating fully connected intra-class graph, Pang et al. \cite{Pang2019} aims at simultaneously learning neighborship and projection matrix (SLNP). Recently, Nie et al. \cite{Nie2019} put forward submanifold-preserving discriminant analysis (SPDA) with an auto-optimized k-nearest neighbor (KNN) graph, which differs from SLNP in that it considers only the connection of nearby samples. The data manifold has the property of local Euclidean homeomorphism, which makes the partial connected graph superior to the fully connected graph in revealing the manifold structure. Besides, considering only the connection information and not the affinity between samples, the model is easily affected by outliers.

Aside from preserving the intra-class structure, inter-class separability is also an indispensable part of achieving better classification performance. Establishing an effective margin representation facilitates the separability between classes. Here, we classify the commonly-used margin representations in the current mainstream DR strategies or classifiers into four categories, namely, average margin, weighted average margin, maximum margin and induced margin. Maximizing the inter-class scatter or global scatter in LDA and some of its variants is the pursuit of average margin. Some DR methods based on manifold learning, such as LSDA and stable orthogonal local discriminant embedding (SOLDE) \cite{Gao2013}, with the construction of inter-class graph, aim to achieve the weighted average margin. The maximum margin is typically used in the design of classifiers, such as the maximum margin hyperplane determined by the support vectors in support vector machines (SVMs) \cite{Brereton2010,Hsu2002}. Besides, least-squares regression (LSR), as a fundamental tool in statistics, can also be regarded as a strategy of margin representation. The purpose of margin representation learning can be achieved by guiding the samples of different classes towards disparate pre-set targets. Here, we define such an margin as the induced margin. Over the past decades, various regression analysis methods have been developed, such as ridge regression (RR) \cite{Hoerl1970}, lasso regression \cite{Tibshirani1996}, elastic net regression \cite{Zou2005}, generalized robust regression (GRR) \cite{Lai2019}, adaptive locality preserving regression (ALPR) \cite{Wen2020} and some kernel based regression methods \cite{Rosipal2001,Liu2009,Chen2016}. Most of these methods take the zero-one labels as the regression targets. However, the strict zero-one targets are too harsh on the marginal representation to yield superior classification performance. To remedy this deficiency, the trend is to learn relaxed regression targets instead of the original zero-one targets, with some representative methods such as discriminative LSR (DLSR) \cite{Xiang2012}, retargeted LSR (ReLSR) \cite{Zhang2015}, and groupwise ReLSR \cite{Wang2018}. Among them, ReLSR has been used for the marginally structured representation learning (MSRL) \cite{Zhang2018} and has been successful.

In real life, the collected data may be doped with some noise or outliers. However, in the conventional FE methods, the squared $L_2$ norm, which tends to enlarge the influence of outliers, is generally employed to measure the loss. Of course, we can mitigate this risk by measuring angles instead of distances, as the angle linear discriminant embedding (ALDE) \cite{Liu2015} does. Besides, to suppress the sensitivity of squared $L_2$ norm, some new evaluation criteria based on the $L_1$ norm are combined on PCA and LDA, including $L_{1}$-PCA \cite{Ke2005}, $R_1$-PCA \cite{Ding2006}, PCA-$L_{1}$ \cite{Kwak2008}, LDA-$R_1$ \cite{Li2010}, sparse discriminant analysis (SDA) \cite{Clemmensen2011}, LDA-$L_1$ \cite{Zhong2013,Wang2014}, etc. Since the optimization of $L_1$-norm-based loss function is relatively troublesome, the more efficient $L_{2,1}$ norm gradually attracts the attention of researchers. By imposing the $L_{2,1}$ norm on both the regression loss term and the regularization term, Nie et al. \cite{Nie2010} proposed an efficient and robust feature selection method (RFS). Inspired by this, some new formulations of PCA and LDA with $L_{2,1}$ norm have been proposed successively, and joint sparse PCA (JSPCA) \cite{Yi2017}, adaptive weighted sparse PCA (AW-SPCA) \cite{Yi2019}, L21SDA \cite{Shi2014} and RLDA \cite{Zhao2019} are the representative ones. Unfortunately, RLDA only employs the $L_{2,1}$ norm for the measure of intra-class scatter, thus resulting in limited effect on suppressing outliers.

Taken together, preserving the intrinsic structure of the data manifold and forcing the separation between classes are crucial to generalization and classification performance of feature extraction. Combining with the excellent properties of $L_{2,1}$ norm on resisting outliers and removing redundant features, we propose a novel discriminant feature extraction algorithm with flexible learning of intra-class structure and margin representation, which has the following advantages.
\begin{enumerate}
  \item By integrating locality-aware graph learning and flexible margin representation learning, we build a new discriminant learning criterion, which enhances intra-class compactness while allowing for flexible intra-class and inter-class differences.
  \item The flexible locality-aware structural learning strategy formulated by us is capable of revealing the local adjacency structure of the intra-class samples in the desired subspace.
  \item Joint $L_{2,1}$ norms in the proposed model can not only relieve the pressure brought by outliers or noise, but also conduct feature selection and subspace learning simultaneously.
  \item We theoretically prove the convergence of the proposed algorithm and experimentally verify its superior classification and generalization performance on multiple databases.
\end{enumerate}

The remainder of our paper is organized as follows. In Section \ref{related_work}, we briefly describe some notations and review some of the work. In Section \ref{methodology}, we elaborate the process of establishing the model and derive an effective algorithm. In Section \ref{algorithm_analysis}, we provide theoretical analysis of the proposed algorithm, including convergence proof and computational complexity analysis. In Section \ref{experiments}, we investigate the performance of the proposed algorithm through a series of comparative experiments. Section \ref{conclusion} concludes the paper with some additional summary.

\section{Related Work}\label{related_work}
Briefly, some notations in our writing are described in this section. Throughout, all the vectors and matrices we defined are in bold italics, and all other variables are in single italics. Given a data matrix $\bm X\!=\![\bm X_1,\bm X_2,\cdots,\bm X_n]\!\in\!\mathbb{R}^{d\times n}$, where $\bm X_i$ refers to a sample instance. Assuming that $\bm X$ can be classified into $c$ classes, we write $\bm X^i\!=\![\bm X^i_1,\bm X^i_2,\cdots,\bm X^i_{n_i}]\!\in\!\mathbb{R}^{d\times n_i}$ as the data submatrix of class $i$. Moreover, $\bm X^T$, $\bm X^{-1}$ and $Tr(\bm X)$ represent $\bm X$'s transpose, inverse and trace, respectively. And, we denote the matrix $[\bm A;\bm B]\!\in\!\mathbb{R}^{(p+q)\times t}$ as the composition of $\bm A$ and $\bm B$, where $\bm A\!\in\!\mathbb{R}^{p\times t}$ and $\bm B\!\in\!\mathbb{R}^{q\times t}$. Some commonly-used norms, such as the Frobenius norm, $L_2$ norm, and $L_{2,1}$ norm, are defined as $\|\bm X\|_F=\sqrt{\sum_{i,j}X_{ij}^2}$, $\|\bm X_{i,:}\|_2=\sqrt{\sum_jX_{ij}^2}$, and $\|\bm X\|_{2,1}=\sum_i\|\bm X_{i,:}\|_2$, respectively. Using these notations, we redescribe several of the work related to our research below.
\subsection{LDA}
LDA projects the high-dimensional data $\bm X$ into the low-dimensional latent space through a linear mapping $f(\bm X)=\bm W^T\bm X$, which aggregates the intra-class samples to the intra-class mean while maximizing the discrepancy between the inter-class means as follows
\begin{equation}\label{eq:LDA}
\max_{\bm W}\frac{\sum_{i=1}^cn_i\left\|\bm W^T(\bm M_i-\bm M)\right\|_2^2}{\sum_{i=1}^c\sum_{j=1}^{n_i}\left\|\bm W^T(\bm X^i_j-\bm M_i)\right\|_2^2}
\end{equation}
where $\bm M_i=\frac{1}{n_i}\sum_{j=1}^{n_i}\bm X^i_j$ refers to the intra-class mean of class $i$ and $\bm M\!=\!\frac{1}{n}\sum_{i=1}^n\bm X_i$ denotes the mean of all the samples.

Denoting the intra-class scatter matrix and inter-class scatter matrix as $\bm S_w\!=\!\sum_{i=1}^c\sum_{j=1}^{n_i}(\bm X^i_j-\bm M_i)(\bm X^i_j-\bm M_i)^T$ and $\bm S_b\!=\!\sum_{i=1}^cn_i(\bm M_i-\bm M)(\bm M_i-\bm M)^T$ respectively, we can rewrite Eq.~(\ref{eq:LDA}) as the following trace ratio problem
\begin{equation}\label{eq:LDA:trace}
  \max_{\bm W}\frac{Tr(\bm W^T\bm S_b\bm W)}{Tr(\bm W^T\bm S_w\bm W)}.
\end{equation}
Solving problem (\ref{eq:LDA:trace}) is equivalent to solving the following generalized eigendecomposition problem
\begin{equation}\label{eq:LDA:solution}
  \bm S_b\bm W=\bm S_w\bm W\bm \Lambda.
\end{equation}
Since Eq.~(\ref{eq:LDA:solution}) involves matrix inverse, LDA suffers from the small-sample-size problem \cite{Raudys1991}. From the definitions of intra-class scatter and inter-class scatter, LDA treats each sample equally, which causes the samples far from the mean to have a greater impact on the model. Moreover, LDA can handle the non-gaussian data \cite{Li2017,Luo2018} hardly because of the mean dependence.
\subsection{LSR and ReLSR}
Here, we briefly review the classical least squares regression (LSR) model with a class indicator matrix $\bm Y\!=\![\bm Y_1,\bm Y_2,\cdots,\bm Y_n]^T\!\in\!\mathbb{R}^{n\times c}$ which is assigned to the data matrix $\bm X$. Assuming that the linear mapping between the row vectors in $\bm Y$ and the column vectors in $\bm X$ is $\bm Y_i=\bm W^T\bm X_i+\bm b$ where $\bm W\in\mathbb{R}^{d\times c}$ refers to the regression matrix and $\bm b\in\mathbb{R}^{c\times 1}$ is a bias vector, we can obtain unbiased estimates of $\bm W$ and $\bm b$ by solving the following objective function
\begin{equation}\label{eq:LSR}
  \min_{\bm W,\bm b}\sum_{i=1}^n\left\|\bm W^T\bm X_i+\bm b-\bm Y_i\right\|_2^2.
\end{equation}
Conventionally, the indicator matrix $\bm Y$ is a strict zero-one matrix in which only the $l_i$-th entry of each row is one, where $l_i\!\in\!\{1,2,\cdots,c\}$ is the class label of sample $\bm X_i$. In reality, however, due to the diversity of data sampled from various distributions, strict zero-one indicators do not make sense and may be detrimental to classification.

To overcome this problem, Zhang et al. proposed the retargeted least squares regression (ReLSR) \cite{Zhang2015}, a method of learning targets flexibly from regression results, while maintaining a certain discriminant power. The joint learning framework of linear mapping and regression target of ReLSR is as follows
\begin{equation}\label{eq:ReLSR}
\begin{split}
  &\min_{\bm W,\bm b,\bm T}\left\|\bm X^T\bm W+\bm 1_n\bm b^T-\bm T\right\|_F^2+\beta\|\bm W\|_F^2\\
  &s.t.~T_{il_i}-\max_{j\neq l_i}T_{ij}\geq 1, i=1,2,\cdots,n
\end{split}
\end{equation}
where $\bm 1_n\!\in\!\mathbb{R}^{n\times 1}$ refers to a column vector with all 1s, $\bm T\in\mathbb{R}^{n\times c}$ represents a retargeted matrix and $\beta$ is a regularization parameter that controls the degree of bias. Actually, (\ref{eq:ReLSR}) can be regarded as a ridge regression (RR) \cite{Hoerl1970} of adaptive learning regression targets with a margin constraint on the true class and the most likely false class. By solving problem (\ref{eq:ReLSR}), we can achieve a more accurate classification than solving problem (\ref{eq:LSR}).
\subsection{RFS}
Besides the inflexible predefined targets, another drawback of LSR is that its loss function is in the form of squared Frobenius norm, which inevitably results in sensitivity to outliers or noise. To address this issue, Nie et al. proposed an efficient and robust feature selection (RFS) method \cite{Nie2010}, which avoids the dilemma by jointly minimizing the $L_{2,1}$-norms. The objective function of RFS with our notations can be written as follows
\begin{equation}\label{eq:RFS}
  \min_{\tilde{\bm W}}\frac{1}{\gamma}\left\|\tilde{\bm X}^T\tilde{\bm W}-\bm Y\right\|_{2,1}+\|\tilde{\bm W}\|_{2,1}
\end{equation}
where $\tilde{\bm W}$ absorbing the offset vector $\bm b$ is denoted as $[\bm b^T;\bm W]$, and correspondingly $\tilde{\bm X}\!=\![\bm 1_n^T;\bm X]$. It can be seen from Eq.~(\ref{eq:RFS}) that simultaneously utilizing $L_{2,1}$-norms on both loss function and regularization term can not only resist outliers, but also jointly induce the row sparsity of regression matrix.

With $\bm E=\frac{1}{\gamma}(\tilde{\bm X}^T\tilde{\bm W}-\bm Y)$, $\bm A=[\tilde{\bm X}^T~\gamma \bm I]$ and $\bm U=[\bm W;\bm E]$, rewriting the above problem (\ref{eq:RFS}) as
\begin{equation}\label{eq:re:RFS}
  \min_{\bm U}\|\bm U\|_{2,1}~~~~s.t.~~~\bm A\bm U=\bm Y.
\end{equation}
The constrained problems similar to Eq.~(\ref{eq:re:RFS}) can easily derive a closed-form solution by using the Lagrange multiplier method.
\section{Methodology}\label{methodology}
In this section, we analyze the irrationality of LDA optimization criteria, and establish more discriminative optimization criteria on the premise of ensuring higher generalization performance by replacing intra-class scatter and inter-class scatter. Besides, to mitigate the interference caused by outliers or noises, we focus on minimizing the joint $L_{2,1}$-norms of all modules in our model. Then, we deduce the process of model optimization and condense an effective algorithm.
\subsection{Proposed Model}
As it involves the performance of classification and generalization, there are two criteria for evaluating the quality of supervised dimensionality reduction, one is the maintenance of the intra-class structural information, the other is the preservation of the inter-class separability. Although LDA has both of these criteria, namely minimizing intra-class scatter and maximizing inter-class scatter, there are obvious deficiencies, as analyzed below.

Reformulates the loss function of the intra-class scatter as the form of sample pair as follows
\begin{equation}\label{eq:LDA:intraclass}
\begin{split}
&\sum_{i=1}^c\sum_{j=1}^{n_i}\left\|\bm W^T\left(\bm X^i_j-\bm M_i\right)\right\|_2^2\\
=&\sum_{i=1}^c\sum_{j,k=1}^{n_i}\!S_{jk}^i\!\left\|\bm W^T\bm X^i_j-\bm W^T\bm X^i_k\right\|_2^2
\end{split}
\end{equation}
where $S_{jk}^i=\frac{1}{2n_i}$ refers to the similarity between embedded sample pairs (See appendix \ref{LDA:intra-class} for details of proof). Obviously from Eq.~(\ref{eq:LDA:intraclass}), the fully connected intra-class graph is employed in LDA, and $S^i_{jk}$ of LDA is only connected with the number of intra-class samples, which means that there is no gap in the affinity between all intra-class embedded sample pairs. This is extremely unreasonable, because samples far from the intra-class mean contribute more to the change in the objective function.

Since non-orthogonal linear transformation cannot guarantee metric invariance, the affinity between sample pairs in the ambient high-dimensional space cannot be maintained in the embedded low-dimensional space. Here, we develop an adaptive locality-aware method for calculating the affinity of samples following the learning process of projection matrix, as shown below
\begin{equation}\label{eq:obj:1}
\begin{split}
&\min_{\bm V,\bm W}\sum_{i=1}^c\sum_{j,k=1}^{n_i}\frac{V_{jk}^i}{2K}\left\|\bm W^T\bm X^i_j-\bm W^T\bm X^i_k\right\|_2\\
&s.t.~\bm V_{j,:}^i\bm 1=K,\ V_{jk}^i\in\{0,1\}
\end{split}
\end{equation}
where $V_{jk}^i$ refers to the connection indicator in the $j$-th KNN graph of class $i$. Instead of measuring the loss by the squared $L_2$ norm in (\ref{eq:LDA:intraclass}), we can reduce the interference of outliers by directly using the $L_2$ norm in the modeling of (\ref{eq:obj:1}). Although intuitively only the connections in the intra-class graph are optimized in (\ref{eq:obj:1}), the sample affinity can be induced below.
\begin{equation}\label{eq:obj:11}
\begin{split}
&\min_{\bm S,\bm W}\sum_{i=1}^c\sum_{j,k=1}^{n_i}S_{jk}^i\left\|\bm W^T\bm X^i_j-\bm W^T\bm X^i_k\right\|_2^2\\
&s.t.~\|\bm S_{j,:}^i\|_0=K
\end{split}
\end{equation}
where $\|\cdot\|_0$ denotes the number of non-zero entries in a vector, and $S_{jk}^i\!=\!\frac{1}{2K\|\bm W^T\bm X^i_j-\bm W^T\bm X^i_k\|_2}$ tends to penalize sample pairs that are far away from each other, which can be regarded as a significative definition of affinity between samples.

In terms of inter-class separability, LDA requires that the average margin between different classes should be maximally expanded, and classes far away from other classes will occupy a larger proportion. In other words, the definition of such inter-class separability makes different classes influenced by each other easily. The alternative strategy to eliminate this influence is to employ the margin representation in the form of regression. In the original regression, the strict zero-one target matrix cannot be approximated as an ideal low-dimensional embedding, thus we prefer the flexible formulation in problem (\ref{eq:ReLSR}) which helps to realize the learning of margin representation. Combining the retargeted regression term and locality-aware term, we have the following learning model
\begin{equation}\label{eq:obj:2}
\begin{split}
&\min_{\bm W,\bm b,\bm V,\bm T}\left\|\bm X^T\bm W+\bm 1\bm b^T-\bm T\right\|_{F}^2\\
&~~~~~~~~~~~~~~+\beta\sum_{i=1}^c\sum_{j,k=1}^{n_i}S_{jk}^i\left\|\bm W^T\bm X_{j}^i-\bm W^T\bm X_{k}^i\right\|_2^2\\
&s.t.~T_{il_i}-\max_{j\neq l_i}T_{ij}\geq1,\|\bm S_{j,:}^i\|_0=K
\end{split}
\end{equation}
where $\beta\!>\!0$ is a tradeoff coefficient.

In real life, the collected data often have higher dimensions and are mixed with various noises, which leads to the generation of many outliers and the appearance of redundant features. To alleviate the interference of outliers on the training process, we tend to employ $L_2$ norm instead of squared $L_2$ norm to measure the value of loss function. As indicated by problem (\ref{eq:RFS}), for matrix variables, it should be $L_{2,1}$ norm instead of Frobenius norm. Besides, the projection matrix plays the role of feature loading and feature fusion. We can achieve the intention of feature selection by forcing the row sparsity of the projection matrix, which can be achieved by performing the $L_{2,1}$ norm \cite{Nie2010,Lai2019,Wen2020}. To sum up, we establish a unified learning framework, which covers the joint learning of adaptive graph structure, projection matrix with feature selection and margin representation. The objective function is as follows
\begin{equation}\label{eq:obj}
\begin{split}
&\min_{\bm W,\bm b,\bm V,\bm T}\left\|\bm X^T\bm W+\bm 1\bm b^T-\bm T\right\|_{2,1}+\alpha\left\|\bm W\right\|_{2,1}\\
&~~~~~~~~~~~~~+\beta\sum_{i=1}^c\sum_{j,k=1}^{n_i}\frac{V_{jk}^i}{2K}\left\|\bm W^T\bm X_{j}^i-\bm W^T\bm X_{k}^i\right\|_2\\
&s.t.~T_{il_i}-\max_{j\neq l_i}T_{ij}\geq1,\bm V_{j,:}^i\bm 1=K,\ V_{jk}^i\in\{0,1\}
\end{split}
\end{equation}
where $\alpha\!>\!0$ is a regularization penalty parameter.
\subsection{Optimization Strategy}
Obviously, since the regression term, regularization term and locality-aware term in (\ref{eq:obj}) are all characterized by $L_{2,1}$ norm, problem (\ref{eq:obj}) is a non-smooth optimization problem with multivariable coupling, which urgently needs us to obtain the optimal solution through the strategy of alternating iteration. In each iteration, we transform the problem into an smooth optimization problem that is jointly convex for all variables. The details are as follows.

Fix $\bm V$ and $\bm T$, Update $\bm W$ and $\bm b$: First, we re-formulate the problem (\ref{eq:obj}) with Frobenius norm as
\begin{equation}\label{eq:re:obj}
\begin{split}
&\min_{\bm W,\bm b,\bm V,\bm T}\left\|\sqrt{\hat{\bm D}}\left(\bm X^T\bm W\!+\!\bm 1\bm b^T\!-\!\bm T\right)\right\|_F^2\!+\!\alpha\left\|\sqrt{\tilde{\bm D}}\bm W\right\|_F^2\\
&~~~~~~~~~~~+\beta\frac{1}{2}\sum_{i=1}^c\sum_{j,k=1}^{n_i}\frac{V_{jk}^i}{G_{jk}^i}\left\|\bm W^T\bm X_{j}^i\!-\!\bm W^T\bm X_{k}^i\right\|_2^2\\
\end{split}
\end{equation}
where $\hat{\bm D}\in\mathbb{R}^{n\times n}$ and $\tilde{\bm D}\in\mathbb{R}^{d\times d}$ are diagonal matrices with $ii$-th entries $1/\|(\bm X^T\bm W+\bm 1\bm b^T-\bm T)_{i,:}\|_2$ and $1/\|\bm W_{i,:}\|_2$ respectively, $G^i_{jk}=K\|\bm W^T\bm X_{j}^i-\bm W^T\bm X_{k}^i\|_2$ and specifically, $\frac{0}{0}=0$. When $\bm S$ and $\bm T$ are known, and $\hat{\bm D}$, $\tilde{\bm D}$, and $\bm G$ are assumed to be constants, Eq.~(\ref{eq:re:obj}) can be written as the matrix trace optimization problem with respect to $\bm W$ and $\bm b$ is as follows
\begin{equation}\label{eq:re:tr:obj}
\begin{split}
&\min_{\bm W,\bm b}Tr\!\left(\!\left(\!\bm X^T\!\bm W\!+\!\bm 1\bm b^T\!-\!\bm T\!\right)^T\!\hat{\bm D}\!\left(\!\bm X^T\!\bm W\!+\!\bm 1\bm b^T\!-\!\bm T\!\right)\!\right)\!\\
&\!+\!\alpha Tr\!\left(\!\bm W^T\!\tilde{\bm D}\!\bm W\!\right)\!+\!\beta Tr\!\left(\!\bm W^T\!\bm X\!(\!\bm D\!-\!\frac{\bm S\!+\!\bm S^T}{2}\!)\!\bm X^T\bm\!\bm W\!\right)\!\\
\end{split}
\end{equation}
where $\bm S\!=\!\bm V\oslash\bm G$ (Note that $\oslash$ is the element-wise division operator of matrices), $\bm D$ is a diagonal matrix with the $i$-th entry $D_{ii}=\sum_j(S_{ij}+S_{ji})/{2}$.

Taking the derivative of Eq.~(\ref{eq:re:tr:obj}) w.r.t. $\bm b$ and setting it to zero, we get
\begin{equation}\label{eq:b}
\begin{split}
  &\bm W^T\bm X\hat{\bm D}\bm 1+\bm b\bm 1^T\hat{\bm D}\bm 1-\bm T^T\hat{\bm D}\bm 1=\bm 0\\
  &\Rightarrow\bm b=\frac{\left(\bm T^T-\bm W^T\bm X\right)\hat{\bm D}\bm 1}{\bm 1^T\hat{\bm D}\bm 1}.
\end{split}
\end{equation}
Then, similarly setting the derivative of $\bm W$ to zero, we arrive at
\begin{equation}\label{eq:der:W}
\begin{split}
  &\bm X\!\hat{\bm D}\!\left(\!\bm X^T\!\bm W\!+\!\bm 1\bm b^T\!-\!\bm T\!\right)\!+\!\alpha\tilde{\bm D}\bm W\!+\!\beta\bm X\bm L\bm X^T\bm W=\bm 0\\
\end{split}
\end{equation}
where $\bm L\!=\!\bm D-\frac{\bm S+\bm S^T}{2}$. Combining Eq.~(\ref{eq:b}) and Eq.~(\ref{eq:der:W}), we obtain the optimal $\bm W$ as follows
\begin{equation}\label{eq:W}
  \bm W = \left(\bm X\bm H\bm X^T\!+\!\alpha\tilde{\bm D}\!+\!\beta\bm X\bm L\bm X^T\!\right)^{-1}\bm X\bm H\bm T
\end{equation}
where $\bm H=\hat{\bm D}-\frac{\hat{\bm D}\bm 1\bm 1^T\hat{\bm D}}{\bm 1^T\hat{\bm D}\bm 1}$.

Fix $\bm W$, $\bm b$ and $\bm T$, Update $\bm V$: Since $\bm W$, $\bm b$ and $\bm T$ are fixed, Eq.~(\ref{eq:obj}) can be reduced to
\begin{equation}\label{eq:subobj:V}
\begin{split}
&\min_{\bm V}\sum_{i=1}^c\sum_{j,k=1}^{n_i}V_{jk}^i\left\|\bm W^T\bm X^i_j-\bm W^T\bm X^i_k\right\|_2\\
&s.t.~\bm V_{j,:}^i\bm 1=K,\ V_{jk}^i\in\{0,1\}
\end{split}
\end{equation}
Eq.~(\ref{eq:subobj:V}) means that each subproblem of $i$ and $j$ is independent of each other. Then, the above problem can be further simplified to
\begin{equation}\label{eq:subobj:V2}
\begin{split}
&\min_{\bm V^i_{j,:}}\bm V^i_{j,:}{\bm G^i_{j,:}}^T~~s.t.~~\bm V_{j,:}^i\bm 1=K,\ V_{jk}^i\in\{0,1\},
\end{split}
\end{equation}
where $\bm 1$ refers to a column vector with all entries 1. Accordingly, the optimal solution of Eq.~(\ref{eq:subobj:V2}) can be directly determined by $K$ non-zero minimum values in vector $\bm G^i_{j,:}$.

Fix $\bm V$, $\bm W$ and $\bm b$, Update $\bm T$: By fixing the regression matrix $\bm W$ and offset $\bm b$, Eq.~(\ref{eq:obj}) degenerates into a retargeting problem \cite{Zhang2015}
\begin{equation}\label{eq:T}
\begin{split}
  &\min_{\bm T}\left\|\bm X^T\bm W+\bm 1\bm b^T-\bm T\right\|_{2,1}=\left\|\bm Y-\bm T\right\|_{2,1}\\
  &s.t.~T_{il_i}-\max_{j\neq l_i}T_{ij}\geq1
\end{split}
\end{equation}
where the regression result $\bm X^T\bm W+\bm 1\bm b^T$ is simply denoted as $\bm Y\in\mathbb{R}^{n\times c}$. As can be seen easily from Eq.~(\ref{eq:T}), there are $n$ mutually independent constrained convex subproblems, each of which is shown below
\begin{equation}\label{eq:sub:T}
\begin{split}
  &\min_{\bm T_{i,:}}\left\|\bm Y_{i,:}-\bm T_{i,:}\right\|_{2}=\sqrt{\sum_{j=1}^c\left(Y_{ij}-T_{ij}\right)^2}\\
  &s.t.~T_{il_i}-\max_{j\neq l_i}T_{ij}\geq1.
\end{split}
\end{equation}

As in \cite{Zhang2015}, we redefined the target vector $\bm T_{i,:}$ as follows
\begin{equation}\label{eq:re:T}
  T_{ij}=\left\{\!\begin{array}{ll}
                  Y_{ij}+\triangle_i,&j=l_i \\
                  Y_{ij}+\min{(\triangle_i-v_j,0)}, &j\neq l_i
                \end{array}
  \right.
\end{equation}
where $\triangle_i=T_{il_i}-Y_{il_i}$ is a step parameter, $v_j=Y_{ij}+1-Y_{il_i}$ is an indicator and $v_j\leq 0$ means that class $i$ and class $l_i$ satisfy the margin constraint. Using the new representation in Eq.~(\ref{eq:re:T}), we can rewrite optimization problem (\ref{eq:sub:T}) as
\begin{equation}\label{eq:new:sub:T}
  \min\Gamma(\triangle_i)=\sqrt{\triangle_i^2+\sum_{j\neq l_i}(\min{(\triangle_i-v_j,0)})^2}.
\end{equation}
Denoting $\tau=\sqrt{\triangle_i^2+\sum_{j\neq l_i}(\min{(\triangle_i-v_j,0)})^2}$ and taking the derivative of Eq.~(\ref{eq:new:sub:T}) as follows
\begin{equation}\label{eq:der:sub:T}
  \Gamma'(\triangle_i)=\frac{1}{\tau}\left(\triangle_i+\sum_{j\neq l_i}\left(\min{(\triangle_i-v_j,0)}\right)\right).
\end{equation}
Obviously, when $\triangle_i+\sum_{j\neq l_i}\left(\min{(\triangle_i-v_j,0)}\right)=0$, problem (\ref{eq:new:sub:T}) minimizes, and the optimal $\triangle_i$ is calculated as
\begin{equation}\label{eq:triangle}
  \triangle_i=\frac{\sum_{j\neq l_i}v_j\Phi(v_j)}{1+\sum_{j\neq l_i}\Phi(v_j)}
\end{equation}
where $\Phi(v_j)=\left\{\!\begin{array}{ll}
                          1,& \Gamma'(v_j)>0\\
                          0,& \text{other}
                        \end{array}
\right.$. Then, the optimal target matrix $\bm T$ can be derived from Eq.~(\ref{eq:re:T}).

Based on the above results, we develop an effective alternative iterative algorithm. The detailed steps are described in Algorithm \ref{alg:Framwork}.
\begin{algorithm}[htpb]
\caption{Our RLAR algorithm of solving problem (\ref{eq:obj})}
\label{alg:Framwork}
\begin{algorithmic}[1]
\Require Data matrix $\bm X\!\in\!\mathbb{R}^{d\!\times\!n}$, labels $\{l_i\}_{i=1}^n$, penalty parameter $\alpha$ and tradeoff coefficient $\beta$.\\
Initialize $\hat{\bm D}\!\in\!\mathbb{R}^{n\!\times\!n}$, $\tilde{\bm D}\!\in\!\mathbb{R}^{d\!\times\!d}$ as the identity matrices.\\
Initialize $G_{jk}^i\!=\!\|\bm X_{j}^i-\bm X_{k}^i\|_2$.\\
Initialize target matrix $T_{ij}\!=\!\left\{\!\begin{array}{ll}
                  1,&j=l_i \\
                  0, &j\neq l_i
                \end{array}\right.$.
\Repeat
\State Update $V^i_{jk}$ by solving the problem (\ref{eq:subobj:V}).
\State Update the affinity matrix $\bm S=\bm V\oslash\bm G$.
\State Calculate $\bm W=\left(\bm X\left(\bm H\!+\!\beta\bm L\right)\bm X^T\!+\!\alpha\tilde{\bm D}\!\right)^{-1}\!\bm X\bm H\bm T$, where $\bm H=\hat{\bm D}-\frac{\hat{\bm D}\bm 1\bm 1^T\hat{\bm D}}{\bm 1^T\hat{\bm D}\bm 1}$ and $\bm L=\!\bm D\!-\!\frac{\bm S+\bm S^T}{2}\!$.
\State Calculate $\bm b=\frac{\left(\bm T^T-\bm W^T\bm X\right)\hat{\bm D}\bm 1}{\bm 1^T\hat{\bm D}\bm 1}$.
\State Calculate $T_{ij}=\left\{\!\begin{array}{ll}
                  Y_{ij}+\triangle_i,&j=l_i \\
                  Y_{ij}+\min{(\triangle_i-v_j,0)}, &j\neq l_i
                \end{array}\right.$, where $\bm Y=\bm X^T\bm W+\bm 1\bm b^T$, $\triangle_i=\frac{\sum_{j\neq l_i}v_j\Phi(v_j)}{1+\sum_{j\neq l_i}\Phi(v_j)}$ and $v_j=Y_{ij}+1-Y_{il_i}$.
\State Update $\hat{D}_{ii}\!=\!\frac{1}{\|(\bm X^T\bm W+\bm 1\bm b^T-\bm T)_{i,:}\|_2+\varepsilon}$\footnotemark[1].
\State Update $\tilde{D}_{ii}\!=\!\frac{1}{\|\bm W_{i,:}\|_2+\varepsilon}$.
\State Update {$G_{jk}^i\overset{j\neq k}{=}\|\bm W^T\bm X_{j}^i-\bm W^T\bm X_{k}^i\|_2$}.
\Until{Convergence}
\Ensure $\bm W$, $\bm b$, $\bm V$, $\bm S$, $\bm T$.
\end{algorithmic}
\end{algorithm}
\footnotetext[1]{As indicated here, positive regularization perturbation $\varepsilon\!\rightarrow\!0$ can be employed to guarantee that the denominator is not zero in practice.}

\section{Algorithm Analysis}\label{algorithm_analysis}
\subsection{Convergence Analysis}
\begin{lemma}\label{lemma:con:ineq}
For any two sets of non-zero constants $\{a_i\}_{i=1}^n$ and $\{b_i\}_{i=1}^n$ where $a_i,b_i\in\mathbb{R}^{+}$, if the following inequality holds
\[
\sum_{i=1}^n\frac{a_i}{\sqrt{b_i}}\leq\sum_{i=1}^n\sqrt{b_i},
\]
then we have $\sum_{i=1}^n\sqrt{a_i}\leq\sum_{i=1}^n\sqrt{b_i}$.
\end{lemma}
\begin{proof}
Combined with Cauchy inequality, the following result can be obtained
\[
\begin{split}
\sum_{i=1}^n\sqrt{a_i}&=\sum_{i=1}^n\frac{\sqrt{a_i}\sqrt{\sqrt{b_i}\sum_{j=1}^n\sqrt{b_j}}}{\sqrt{\sqrt{b_i}\sum_{j=1}^n\sqrt{b_j}}}\\
&\leq\sqrt{\sum_{i=1}^n\frac{a_i}{\sqrt{b_i}\sum_{j=1}^n\sqrt{b_j}}}\sqrt{\sum_{i=1}^n\sqrt{b_i}\sum_{j=1}^n\sqrt{b_j}}.
\end{split}
\]
Since $\sqrt{\sum_{i=1}^n\frac{a_i}{\sqrt{b_i}\sum_{j=1}^n\sqrt{b_j}}}\leq1$ and $\sum_{i=1}^n\sqrt{b_i}>0$, then we have
\[
\sum_{i=1}^n\sqrt{a_i}\leq\sum_{i=1}^n\sqrt{b_i},
\]
which completes the proof.
\end{proof}

\begin{lemma}\label{lemma:ineq}
For any two sets of $z$-dimensional non-zero vectors $\{\bm p_i\}_{i=1}^n$ and $\{\bm q_i\}_{i=1}^n$ where $\bm p_i,\bm q_i\in\mathbb{R}^{z\times 1}$, if the following inequality holds
\[
\sum_{i=1}^nk_i\frac{\|\bm p_i\|^2_2}{\|\bm q_i\|_2}\leq\sum_{i=1}^nk_i\|\bm q_i\|_2, k_i>0
\]
then we have $\sum_{i=1}^nk_i\|\bm p_i\|_2\leq\sum_{i=1}^nk_i\|\bm q_i\|_2$.
\end{lemma}
\begin{proof}
By the definition of norm, we have
\[
\|\bm p_i\|_2=\sqrt{\sum_{j=1}^zp_{ij}^2}, \|\bm q_i\|_2=\sqrt{\sum_{j=1}^zq_{ij}^2}.
\]
Just set $a_i=k_i^2\sum_{j=1}^zp_{ij}^2$, $b_i=k_i^2\sum_{j=1}^zq_{ij}^2$, and then the conclusion in Lemma \ref{lemma:con:ineq} is easy to follow.
\end{proof}
\begin{theorem}\label{convergence_theorem}
Algorithm \ref{alg:Framwork} monotonically decreases the value of the objective function in Eq.~(\ref{eq:obj}) in each iteration, and ultimately converges to the local optimal solution.
\end{theorem}
\begin{proof}
We refer to $F(\bm W^t,\bm b^t,\bm V^t,\bm T^t,\hat{\bm D}^t,\tilde{\bm D}^t,\bm G^t)$ as the objective function of problem (\ref{eq:re:obj}) at $t$-th iteration. Then, by solving subproblem (\ref{eq:re:tr:obj}), we arrive at
\begin{equation}\label{eq:WbT:t+1}
\begin{split}
  &F(\bm W^{t+1},\bm b^{t+1},\bm V^t,\bm T^t,\hat{\bm D}^t,\tilde{\bm D}^t,\bm G^t)\\
  &\leq F(\bm W^t,\bm b^t,\bm V^t,\bm T^t,\hat{\bm D}^t,\tilde{\bm D}^t,\bm G^t).
\end{split}
\end{equation}
Rewriting (\ref{eq:WbT:t+1}) in the form of 2-norms will yield
\begin{equation}\label{eq:Wb:ineq}
\begin{split}
&\sum_{i=1}^n\!\frac{\|(\bm X^T\bm W^{t\!+\!1}\!+\!\bm 1{\bm b^{t\!+\!1}}^T\!-\!\bm T^t)_{i,:}\|_2^2}{\|(\bm X^T\bm W^{t}\!+\!\bm 1{\bm b^{t}}^T\!-\!\bm T^t)_{i,:}\|_2}\!+\!\alpha\sum_{i=1}^d\!\frac{\|\bm W^{t\!+\!1}_{i,:}\|_2^2}{\|\bm W^{t}_{i,:}\|_2}\\
&~~+\beta\sum_{i=1}^c\sum_{j,k=1}^{n_i}\frac{{V_{jk}^i}^t}{2K}\frac{\|{\bm W^{t+1}}^T\bm X_{j}^i-{\bm W^{t+1}}^T\bm X_{k}^i\|_2^2}{\|{\bm W^{t}}^T\bm X_{j}^i-{\bm W^{t}}^T\bm X_{k}^i\|_2}\\
&\leq\sum_{i=1}^n{\|(\bm X^T\bm W^{t}\!+\!\bm 1{\bm b^{t}}^T\!-\!\bm T^t)_{i,:}\|_2}\!+\!\alpha\sum_{i=1}^d{\|\bm W^{t}_{i,:}\|_2}\\
&~~~~+\beta\sum_{i=1}^c\sum_{j,k=1}^{n_i}\frac{{V_{jk}^i}^t}{2K}{\|{\bm W^{t}}^T\bm X_{j}^i\!-\!{\bm W^{t}}^T\bm X_{k}^i\|_2}.
\end{split}
\end{equation}
Obviously, inequality (\ref{eq:Wb:ineq}) satisfies the conditions in Lemma \ref{lemma:ineq}, and then we have
\begin{equation}\label{eq:Wb:ineq:2norm}
\begin{split}
&\sum_{i=1}^n{\|(\bm X^T\bm W^{t\!+\!1}\!+\!\bm 1{\bm b^{t\!+\!1}}^T\!-\!\bm T^t)_{i,:}\|_2}\!+\!\alpha\sum_{i=1}^d{\|\bm W^{t\!+\!1}_{i,:}\|_2}\\
&~~~+\beta\sum_{i=1}^c\sum_{j,k=1}^{n_i}\frac{{V_{jk}^i}^t}{2K}{\|{\bm W^{t+1}}^T\bm X_{j}^i\!-\!{\bm W^{t+1}}^T\bm X_{k}^i\|_2}\\
&\leq\sum_{i=1}^n{\|(\bm X^T\bm W^{t}\!+\!\bm 1{\bm b^{t}}^T\!-\!\bm T^t)_{i,:}\|_2}\!+\!\alpha\sum_{i=1}^d{\|\bm W^{t}_{i,:}\|_2}\\
&~~~+\beta\sum_{i=1}^c\sum_{j,k=1}^{n_i}\frac{{V_{jk}^i}^t}{2K}{\|{\bm W^{t}}^T\bm X_{j}^i\!-\!{\bm W^{t}}^T\bm X_{k}^i\|_2}.
\end{split}
\end{equation}
From (\ref{eq:Wb:ineq:2norm}), we easily know that
\begin{equation}\label{eq:WbDG:t+1}
\begin{split}
  &F(\bm W^{t+1},\bm b^{t+1},\bm V^t,\bm T^t,\hat{\bm D}^{t+1},\tilde{\bm D}^{t+1},\bm G^{t+1})\\
  &\leq F(\bm W^t,\bm b^t,\bm V^t,\bm T^t,\hat{\bm D}^t,\tilde{\bm D}^t,\bm G^t).
\end{split}
\end{equation}
Furthermore, by solving subproblems (\ref{eq:subobj:V}) and (\ref{eq:T}), we obtain the following result
\begin{equation}\label{eq:V:t+1}
\begin{split}
  &F(\bm W^{t+1},\bm b^{t+1},\bm V^{t+1},\bm T^{t+1},\hat{\bm D}^{t+1},\tilde{\bm D}^{t+1},\bm G^{t+1})\\
  &\leq F(\bm W^{t+1},\bm b^{t+1},\bm V^{t+1},\bm T^t,\hat{\bm D}^{t+1},\tilde{\bm D}^{t+1},\bm G^{t+1})\\
  &\leq F(\bm W^{t+1},\bm b^{t+1},\bm V^t,\bm T^t,\hat{\bm D}^{t+1},\tilde{\bm D}^{t+1},\bm G^{t+1}).
\end{split}
\end{equation}
Combined with (\ref{eq:WbDG:t+1}) and (\ref{eq:V:t+1}), the final result holds
\begin{equation}\label{eq:final:t+1}
\begin{split}
  &F(\bm W^{t+1},\bm b^{t+1},\bm V^{t+1},\bm T^{t+1},\hat{\bm D}^{t+1},\tilde{\bm D}^{t+1},\bm G^{t+1})\\
  &\leq F(\bm W^{t},\bm b^{t},\bm V^t,\bm T^{t},\hat{\bm D}^{t},\tilde{\bm D}^{t},\bm G^{t}).
\end{split}
\end{equation}

It is easy to conclude that, following the update step of each variable in Algorithm \ref{alg:Framwork}, the value of the modeled objective function (\ref{eq:obj}) decreases monotonically with the increase of the number of iterations, and finally converges to a local optimal solution.
\end{proof}
\subsection{Computational Complexity Analysis}
The computational complexity of each step of the proposed algorithm RLAR is roughly estimated below. The complexity in step 2 of Algorithm \ref{alg:Framwork} is less than $O(n^2d)$. The calculation of affinity matrix $\bm S$ in step 6 is at most $O(nK)$. Since $d\geq c$, updating $\bm W$ in step 7 is, at most, of order $O(nd^2+d^3)$. Updating $\bm b$ in step 8 takes $O(ndc)$. In step 9, calculating $\bm Y$ costs $O(ndc)$, and then calculating $\bm T$ costs $O(nc)$. It costs $O(nc)$ and $O(dc)$ to calculate $\hat{\bm D}$ and $\tilde{\bm D}$, respectively. Finally, updating $\bm G$ in step 12 is less than $O(n^2c)$. In summary, assuming the algorithm performs $t$ iterations, then the total cost of our RLAR is of order $O(n^2d\!+\!(nd^2\!+\!d^3\!+\!n^2c)t)$ at most. The experimental results show that the RLAR can converge in less iterations. Thus, for large-scale data with lower dimensions, our computational complexity is acceptable.

\section{Experimental Results}\label{experiments}
In this section, we investigate the performance of the proposed RLAR in terms of classification and robustness by comparing it with some state-of-the-art approaches performed on ten publicly benchmark databases. Besides, some visualization results, parameter sensitivity analysis and convergence study are employed to further evaluate the effectiveness of the proposed method.
\begin{table}[htbp]
  \centering
  \caption{Brief description of all databases for classification}
    \begin{tabular}{cccc}
    \hline
    Databases & Instances & Features & Classes \\
    \hline
    Dermatology & 366   & 34    & 6 \\
    Diabetes & 768   & 8     & 2 \\
    Ionosphere & 351   & 34    & 2 \\
    Iris  & 150   & 4     & 3 \\
    Wine  & 178   & 13    & 3 \\
    Binalpha & 1404  & 320   & 36 \\
    YaleB & 2414  & 1024  & 38 \\
    AR    & 1400  & 1024  & 100 \\
    COIL20 & 1440  & 1024  & 20 \\
    Caltech101 & 9144  & 3000  & 102 \\
    \hline
    \end{tabular}%
  \label{tab:data:description}%
\end{table}%
\subsection{Experimental Settings}
The databases involved in our experimental comparisons are from a variety of scenarios to highlight the adaptability of our proposed RLAR, including five UCI data sets and five relatively large-scale databases. Brief information about these databases is described in Table \ref{tab:data:description}. All features of these data are normalized prior to the experiments. And, we repeat each experiment for 10 trials independently with different random splits of training and test data, and then record the mean accuracy and the standard deviation.

All the participating FE algorithms consist of some representative discriminant algorithms, excellent regression algorithms and manifold-inspired algorithms, specifically including RR \cite{Hoerl1970}, LDA \cite{Fisher1936}, MMC \cite{Li2006}, LSDA \cite{Cai2007}, LFDA \cite{Sugiyama2007}, NMMP \cite{Nie2007}, SDA \cite{Clemmensen2011}, SULDA \cite{Zhang2016}, L21SDA \cite{Shi2014}, ALDE \cite{Liu2015}, ReLSR \cite{Zhang2015}, MPDA \cite{Zhou2017}, RSLDA \cite{Wen2019} and ALPR \cite{Wen2020}. And, we perform cross validation to search the best parameters for each algorithm or directly accept the suggested default parameter settings. To be fair, the resulting feature dimension achieved by all algorithms is uniformly set to $c$, except for some LDA-based algorithms that only reduce dimension to $c-1$ at most. For the compact representations produced by running these algorithms, we simply utilize the 1-NN classifier to evaluate the classification performance.

For the number of neighbors used in locality-aware structure learning, we simply set it to 7 for the data split with a training sample size greater than 10 per class and 3 for other cases. Besides, our approach determines two hyper-parameters $\alpha$ and $\beta$ by searching from the grid coordinate set $\{0.001,0.005,0.01,0.05,0.1,0.5,1 ,10,100,1000\}$. And we cover the searching process in details later in the subsection of parameter sensitivity analysis. The termination condition of our algorithm, including all algorithms involved in iterative optimization, is uniformly set to $30$ iterations.

\subsection{Classification Performance Evaluation}
\subsubsection{UCI Classification}
In the experiments conducted here, we employ five small-scale data sets taken from the UCI Machine Learning Repository \footnotemark[2], namely 'Dermatology', 'Diabetes', 'Ionosphere', 'Iris' and 'Wine'.\footnotetext[2]{https://archive.ics.uci.edu/ml/datasets.html} These data sets belong to different domains, which helps verify the universality of our method. We randomly assign 20\% of the samples from these databases to the training set and 80\% to the test set, and repeat the process 10 times. The experimental results are listed in Table \ref{tab:UCI}. And the best results are marked in bold.

Although some comparison methods have achieved superior performance on some data sets, they are not superior in all cases. On these data sets, working well in all of comparisons indicates that our model has strong universality and high efficiency. Meanwhile, achieving a relatively small standard deviation in most cases also suggests that our method is somewhat stable. To some extent, the comparisons with other discriminant methods proves that our method has stronger ability to extract discriminant information.
\begin{table*}[htpb]
  \centering
  \caption{Mean classification accuracy (ACC \%) and standard deviation (Std \%) of various approaches on five UCI data sets}
    \begin{tabular}{|c|c|c|c|c|c|c|c|c|c|c|}
    \hline
    \multirow{2}[4]{*}{Alg.} & \multicolumn{2}{c|}{Dermatology} & \multicolumn{2}{c|}{Diabetes} & \multicolumn{2}{c|}{Ionosphere} & \multicolumn{2}{c|}{Iris} & \multicolumn{2}{c|}{Wine} \\
    \cline{2-11}& ACC  & Std   & ACC  & Std   & ACC  & Std   & ACC  & Std   & ACC  & Std \\
    \hline\hline
    RR    & 95.51  & 1.14  & 55.83  & 1.86  & 86.19  & 2.33  & 96.00  & 1.56  & 64.48  & 5.42  \\
    LDA   & 94.39  & 0.97  & 55.42  & 1.43  & 83.88  & 2.48  & 95.58  & 1.31  & 88.41  & 3.28  \\
    MMC   & 95.78  & 0.87  & 52.26  & 3.05  & 84.63  & 4.10  & 94.08  & 4.79  & 78.21  & 6.51  \\
    LSDA  & 94.39  & 1.08  & 55.83  & 1.51  & 83.70  & 3.80  & 95.17  & 1.66  & 87.93  & 3.38  \\
    LFDA  & 95.14  & 1.09  & 55.60  & 2.98  & 86.58  & 2.15  & 96.50  & 1.35  & 63.59  & 5.28  \\
    NMMP  & 88.61  & 2.88  & 52.84  & 2.25  & 83.70  & 3.49  & 94.42  & 4.29  & 81.03  & 6.47  \\
    SDA   & 93.64  & 1.34  & 51.19  & 2.34  & 74.41  & 7.10  & 95.67  & 1.70  & 63.31  & 5.92  \\
    SULDA & 93.44  & 1.52  & 56.36  & 2.15  & 84.09  & 2.27  & 95.08  & 1.27  & 88.00  & 7.00  \\
    L21SDA & 95.14  & 1.62  & 55.78  & 2.27  & 85.30  & 3.04  & 91.17  & 3.54  & 61.86  & 3.56  \\
    ALDE  & 95.51  & 0.88  & 52.61  & 1.79  & 76.65  & 3.74  & 94.42  & 4.30  & 71.72  & 8.82  \\
    ReLSR & 94.01  & 1.71  & 55.52  & 1.54  & 86.58  & 3.59  & 95.25  & 1.62  & 70.07  & 6.64  \\
    MPDA  & 94.12  & 1.57  & 55.29  & 3.10  & 83.67  & 2.64  & 93.67  & 4.47  & 84.69  & 2.60  \\
    RSLDA & 90.61  & 2.55  & 51.90  & 1.91  & 82.63  & 4.59  & 96.00  & 1.10  & 63.38  & 5.40  \\
    ALPR  & 94.49  & 1.67  & 55.49  & 2.03  & 83.02  & 3.15  & 94.83  & 2.25  & 65.72  & 5.76  \\
    \hline
    RLAR  & \textbf{95.99 } & 0.92  & \textbf{57.00 } & 2.81  & \textbf{86.76 } & 4.42  & \textbf{96.58 } & 1.54  & \textbf{90.34 } & 2.82  \\
    \hline
    \end{tabular}%
  \label{tab:UCI}%
\end{table*}%
\subsubsection{Handwriting Recognition}
Handwriting recognition is one of the classical tasks in pattern recognition and computer vision. To evaluate the performance of the proposed RLAR for this task, we perform a series of comparative experiments on the Binary Alphadigits database \footnotemark[3] to demonstrate the effectiveness of our method.\footnotetext[3]{https://cs.nyu.edu/~roweis/data.html} The database consists of 1404 samples belonging to 36 classes, each of which is a binary image of $20\times16$ pixels. Besides, the database contains not only digits of '0' through '9', but also capital letters of 'A' through 'Z', thus posing a challenge to classification.

For convenience, we simply denote the database as 'Binalpha'. Then, we randomly select $u$ ($u\!=\!10, 13, 16, 19$) images of each subject to form the training set, and the remaining samples to form the test set. The mean recognition results on the database are shown in Table \ref{tab:binalpha}, where '\# number' stands for the number of training samples in each class and is also used in the later recording of experimental results. It can be observed that the recognition rate of each method increases with the expansion of the training set. And, we found that LDA, LSDA, NMMP, SULDA, and MPDA that performed well in the UCI data sets fail on this database. Moreover, our method is superior to many other methods in recognition efficiency, which also indicates that our method has a strong ability of discrimination.
\begin{table}[htbp]
  \centering
  \caption{Classification performance (mean$\pm$std \%) of various approaches on the Binalpha database}
    \begin{tabular}{|c|c|c|c|c|}
    \hline
    Alg. & \# 10 & \# 13 & \# 16 & \# 19 \\
    \hline
    \hline
    RR    & 49.45$\pm$1.10  & 50.68$\pm$0.88  & 52.17$\pm$1.45  & 53.53$\pm$1.14  \\
    LDA   & 11.47$\pm$0.86  & 26.20$\pm$1.71  & 35.40$\pm$1.89  & 41.69$\pm$1.64  \\
    MMC   & 63.65$\pm$1.23  & 64.54$\pm$0.92  & 65.01$\pm$1.46  & 65.89$\pm$1.31  \\
    LSDA  & 11.59$\pm$0.97  & 26.15$\pm$1.81  & 35.40$\pm$1.88  & 41.65$\pm$1.61  \\
    LFDA  & 64.20$\pm$1.05  & 66.22$\pm$0.90  & 67.57$\pm$1.19  & 68.35$\pm$1.28  \\
    NMMP  & 23.14$\pm$2.56  & 31.18$\pm$2.65  & 35.89$\pm$2.54  & 37.01$\pm$2.10  \\
    SDA   & 51.70$\pm$1.34  & 56.05$\pm$0.88  & 59.32$\pm$1.82  & 62.61$\pm$1.94  \\
    SULDA & 10.99$\pm$0.73  & 19.74$\pm$1.34  & 26.30$\pm$1.75  & 30.15$\pm$1.30  \\
    L21SDA & 43.79$\pm$1.39  & 44.51$\pm$1.39  & 46.50$\pm$1.39  & 49.10$\pm$1.67  \\
    ALDE  & 64.64$\pm$1.13  & 66.55$\pm$1.01  & 67.24$\pm$1.31  & 68.39$\pm$1.57  \\
    ReLSR & 53.66$\pm$1.11  & 54.64$\pm$1.15  & 55.69$\pm$1.02  & 56.72$\pm$1.12  \\
    MPDA  & 8.40$\pm$3.15  & 28.29$\pm$2.57  & 40.64$\pm$1.69  & 47.29$\pm$1.42  \\
    RSLDA & 16.21$\pm$1.60  & 27.05$\pm$2.17  & 35.89$\pm$1.41  & 42.51$\pm$1.84  \\
    ALPR  & 51.80$\pm$1.70  & 52.04$\pm$1.42  & 53.45$\pm$1.60  & 55.74$\pm$0.97  \\
    \hline
    RLAR  & \textbf{64.80$\pm$1.31}  & \textbf{66.92$\pm$0.85}  & \textbf{68.33$\pm$1.22}  & \textbf{69.40$\pm$0.87}  \\
    \hline
    \end{tabular}%
  \label{tab:binalpha}%
\end{table}%
\subsubsection{Face Recognition}\label{face:recognition}
In this recognition scenario, we employ two real commonly-used face databases to evaluate the performance of all algorithms, namely the extended YaleB database \cite{Georghiades2001} and the AR database \cite{Martinez1998}. The YaleB database contains 2414 samples from 38 subjects, while the AR database contains more than 4000 color face images of 126 individuals. These two databases are collected under the condition of illumination and expression changes, while the AR database also contains some occlusions. These changes are challenging the performance of our RLAR.

For the AR database, we extract a subset of 1400 images without any occlusion, including 50 female and 50 male subjects, for the experiments. Before implementing all the algorithms, the face images in both of databases are resized to $32\times32$ pixels. Each experiment is independently repeated for 10 times, and the average experimental results of various methods on the two databases are listed in Tables \ref{tab:YaleB} and \ref{tab:AR} respectively, in which the number of training samples in each class is included. And the best results in each set of comparisons are shown in bold. Although our method is slightly inferior to ALPR in Table \ref{tab:YaleB}, it is superior to the others. Moreover, in Table \ref{tab:AR} our approach trumps all others. These are sufficient to confirm that the discriminant model we have established is efficient enough to yield desirable face recognition results.
\begin{table}[htbp]
  \centering
  \caption{Classification performance (mean$\pm$std \%) of various approaches on the YaleB database}
    \begin{tabular}{|c|c|c|c|c|}
    \hline
    Alg. & \# 15 & \# 20 & \# 25 & \# 30 \\
    \hline
    \hline
    RR    & 93.63$\pm$0.99  & 95.87$\pm$0.66  & 97.22$\pm$0.58  & 98.21$\pm$0.60  \\
    LDA   & 90.68$\pm$0.83  & 89.99$\pm$0.94  & 81.91$\pm$0.53  & 86.51$\pm$0.97  \\
    MMC   & 92.44$\pm$1.23  & 94.90$\pm$0.74  & 96.15$\pm$0.74  & 97.31$\pm$0.59  \\
    LSDA  & 90.74$\pm$0.74  & 89.93$\pm$1.04  & 82.38$\pm$0.75  & 86.54$\pm$0.97  \\
    LFDA  & 88.45$\pm$1.18  & 90.56$\pm$0.94  & 91.81$\pm$0.75  & 93.08$\pm$1.00  \\
    NMMP  & 92.39$\pm$0.67  & 92.26$\pm$0.89  & 87.15$\pm$0.58  & 88.04$\pm$0.97  \\
    SDA   & 93.84$\pm$1.04  & 95.66$\pm$0.80  & 96.66$\pm$0.61  & 97.76$\pm$0.30  \\
    SULDA & 89.50$\pm$0.81  & 89.50$\pm$1.04  & 81.71$\pm$0.72  & 85.33$\pm$0.84  \\
    L21SDA & 94.67$\pm$0.81  & 96.84$\pm$0.59  & 97.72$\pm$0.45  & 98.75$\pm$0.26  \\
    ALDE  & 78.80$\pm$1.73  & 83.88$\pm$1.41  & 87.04$\pm$1.00  & 89.40$\pm$0.82  \\
    ReLSR & 94.16$\pm$0.86  & 96.41$\pm$0.69  & 97.55$\pm$0.65  & 98.48$\pm$0.40  \\
    MPDA  & 92.41$\pm$0.67  & 92.27$\pm$0.87  & 87.53$\pm$0.56  & 85.93$\pm$1.06  \\
    RSLDA & 86.20$\pm$1.27  & 89.67$\pm$0.68  & 91.81$\pm$0.74  & 93.41$\pm$0.67  \\
    ALPR  & \textbf{95.70$\pm$0.93}  & \textbf{97.52$\pm$0.52}  & \textbf{98.47$\pm$0.39}  & \textbf{99.16$\pm$0.32}  \\
    \hline
    RLAR  & 95.07$\pm$0.80  & 97.04$\pm$0.72  & 97.84$\pm$0.48  & 98.65$\pm$0.33  \\
    \hline
    \end{tabular}%
  \label{tab:YaleB}%
\end{table}%
\begin{table}[htbp]
  \centering
  \caption{Classification performance (mean$\pm$std \%) of various approaches on the AR database}
    \begin{tabular}{|c|c|c|c|c|}
    \hline
    Alg. & \# 5 & \# 6 & \# 7 & \# 8 \\
    \hline
    \hline
    RR    & 95.07$\pm$0.73  & 96.64$\pm$0.42  & 97.74$\pm$0.67  & 98.18$\pm$0.75  \\
    LDA   & 90.83$\pm$1.12  & 90.98$\pm$0.41  & 90.63$\pm$0.82  & 88.83$\pm$1.54  \\
    MMC   & 91.71$\pm$0.74  & 94.15$\pm$0.75  & 95.57$\pm$0.95  & 96.55$\pm$0.72  \\
    LSDA  & 91.13$\pm$0.92  & 91.21$\pm$0.67  & 90.34$\pm$0.65  & 88.88$\pm$1.40  \\
    LFDA  & 81.71$\pm$1.03  & 85.20$\pm$0.63  & 88.17$\pm$0.83  & 90.47$\pm$1.04  \\
    NMMP  & 94.56$\pm$0.95  & 94.84$\pm$0.39  & 94.94$\pm$0.84  & 94.50$\pm$0.74  \\
    SDA   & 95.44$\pm$0.71  & 97.00$\pm$0.57  & 97.97$\pm$0.59  & 98.12$\pm$0.71  \\
    SULDA & 90.34$\pm$1.05  & 90.34$\pm$0.59  & 90.19$\pm$0.84  & 88.10$\pm$1.48  \\
    L21SDA & 95.11$\pm$0.74  & 97.09$\pm$0.53  & 97.90$\pm$0.50  & 98.17$\pm$0.62  \\
    ALDE  & 92.60$\pm$0.59  & 94.69$\pm$0.64  & 95.76$\pm$0.91  & 96.68$\pm$0.80  \\
    ReLSR & 94.54$\pm$0.66  & 96.35$\pm$0.47  & 97.41$\pm$0.65  & 97.97$\pm$0.75  \\
    MPDA  & 93.84$\pm$0.90  & 94.89$\pm$0.76  & 95.81$\pm$0.86  & 96.03$\pm$1.00  \\
    RSLDA & 87.16$\pm$1.93  & 90.68$\pm$1.48  & 93.01$\pm$1.15  & 94.73$\pm$0.75  \\
    ALPR  & 96.07$\pm$0.70  & 97.78$\pm$0.55  & 98.47$\pm$0.71  & 98.88$\pm$0.48  \\
    \hline
    RLAR  & \textbf{97.33$\pm$0.67}  & \textbf{98.25$\pm$0.30}  & \textbf{98.71$\pm$0.47}  & \textbf{98.95$\pm$0.46}  \\
    \hline
    \end{tabular}%
  \label{tab:AR}%
\end{table}%

Alternatively in Fig. \ref{fig:Visualization}, we illustrate the optimal visualization results for the case of 8 training samples per class on the AR database, including the retargeted matrix $\bm T$, the connection matrix $\bm V$, and the induced affinity matrix $\bm S$. It also shows intuitively that our method has obvious effect on revealing the intra-class local graph structure and learning the target representation with distinct margins.
\begin{figure}[htpb]
  \centering
  \subfigure[]{\label{fig:Visualization:a}
  \centering
  \includegraphics[scale=0.25]{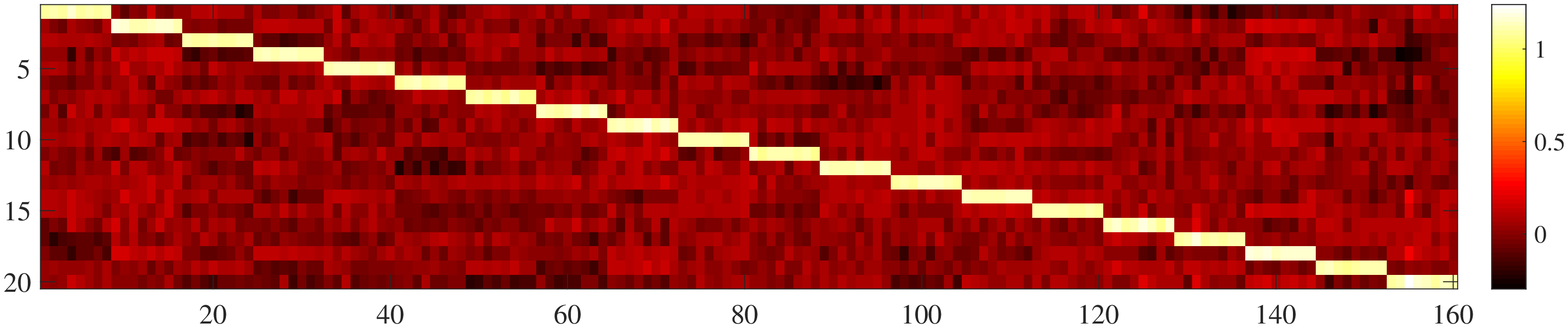}}
  \subfigure[]{\label{fig:Visualization:b}
  \centering
  \includegraphics[scale=0.21]{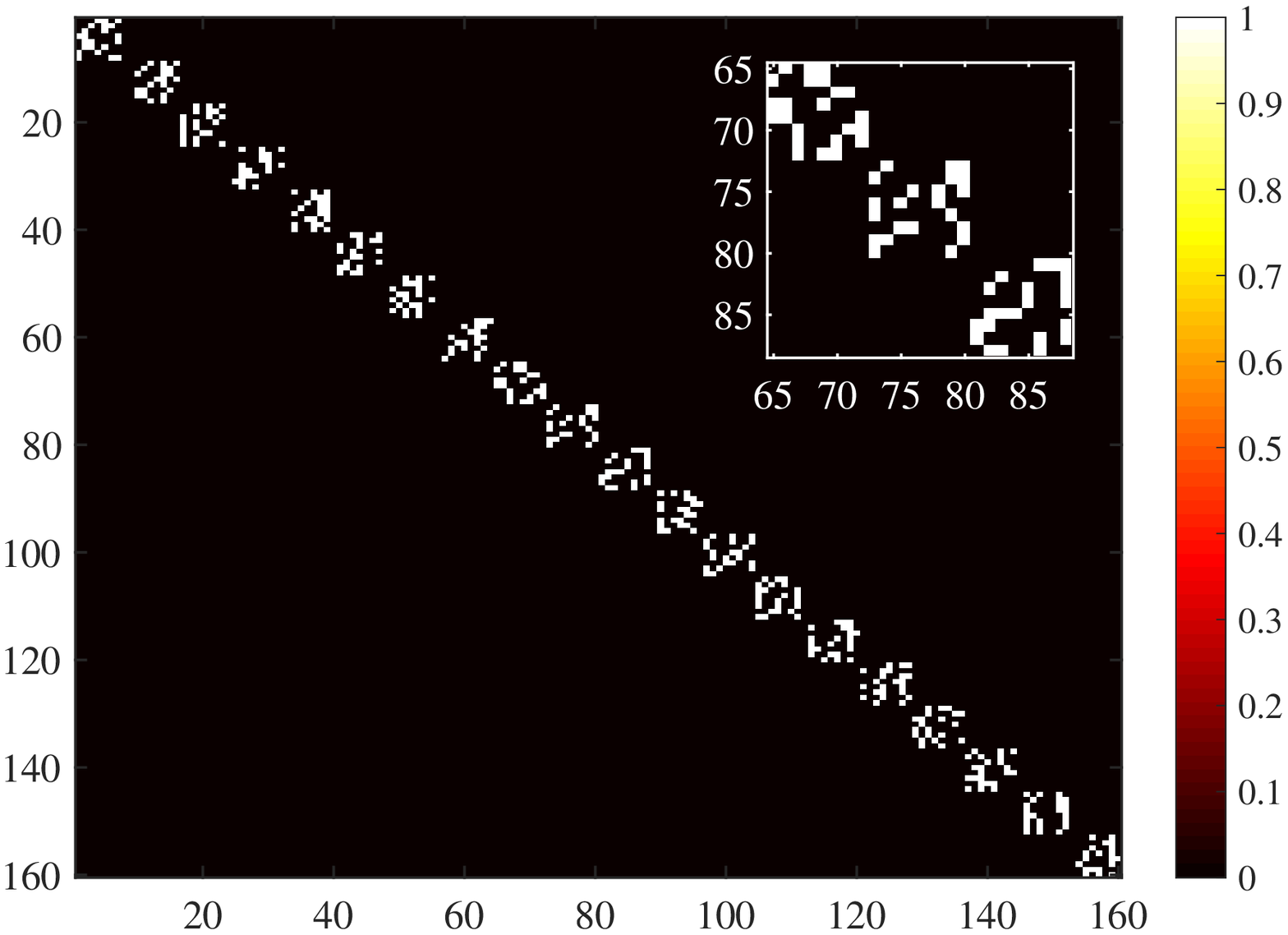}}
  \subfigure[]{\label{fig:Visualization:c}
  \centering
  \includegraphics[scale=0.21]{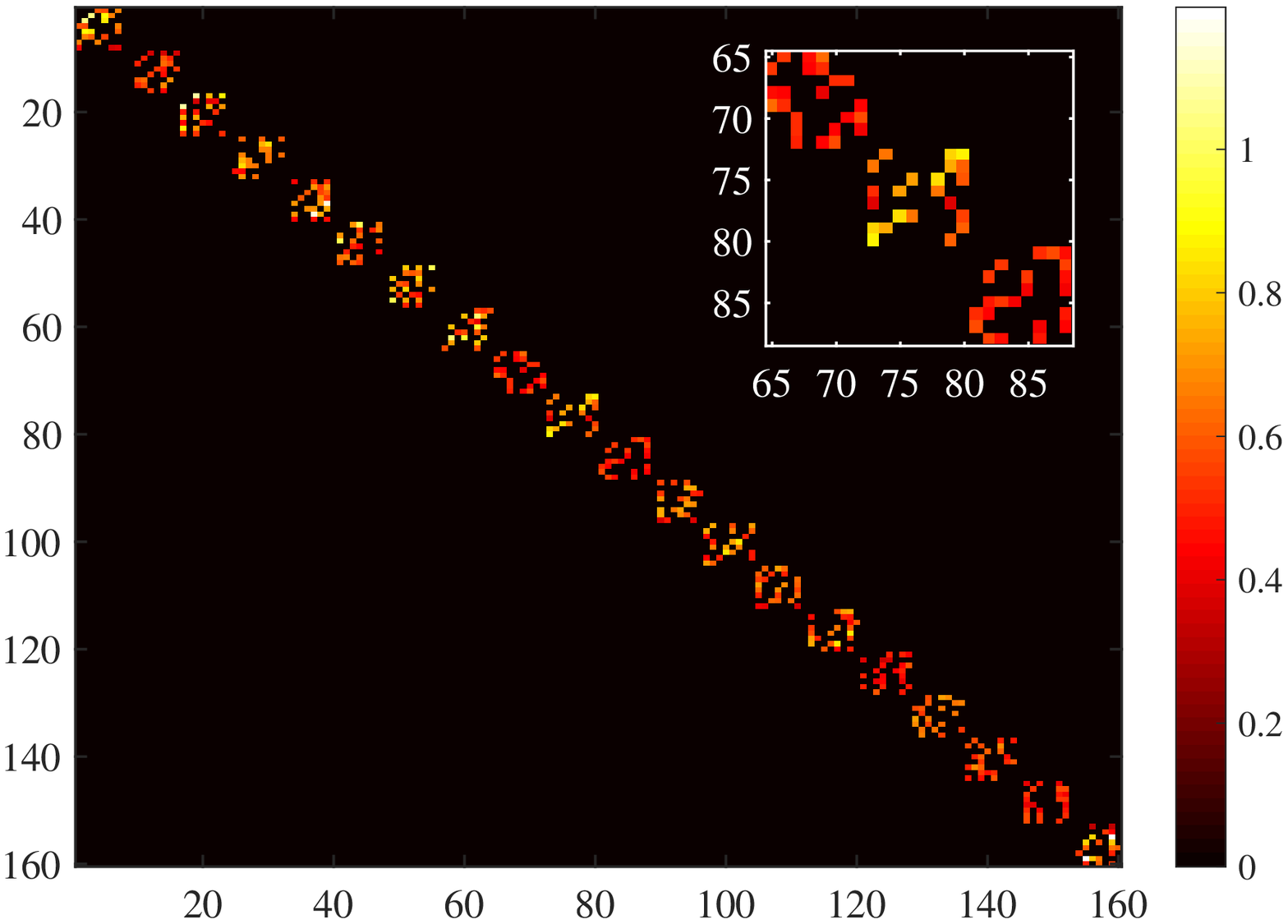}}
  \caption{Visualization results for the top twenty classes on the AR database. \subref{fig:Visualization:a} Retargeted matrix $\bm T$. \subref{fig:Visualization:b} Connection matrix $\bm V$. \subref{fig:Visualization:c} Induced affinity matrix $\bm S$.}\label{fig:Visualization}
\end{figure}
\subsubsection{Object Recognition}
To demonstrate the effectiveness of our method in dealing with the problem of object recognition, we conduct a series of comparison experiments on the COIL20 \footnotemark[4] and Caltech101 \footnotemark[5] databases.\footnotetext[4]{http://www.cs.columbia.edu/CAVE/software/softlib/coil-20.php}\footnotetext[5]{http://www.vision.caltech.edu/Image\_Datasets/Caltech101} The COIL20 database contains 20 objects and a total of 1440 images. As the objects rotate on the turntable, images of each object are taken at 5-degree intervals, with 72 images per object. The Caltech101 database has images of 102 classes of objects containing a background class, each of which has about 40 to 800 images, and most classes have about 50 images.

For the COIL20 database we just use these gray-scale images that are resized to $32\times32$ pixels. For the samples in Caltech101 database, we employ spatial pyramid features with dimension 3000 for recognition in view of the differences in background, size and scale. Besides, $u$ ($u\!=\!15,20,25,30$) and $v$ ($v\!=\!10,15,20,25$) samples are selected from each class of the two databases as the training sets, and the remaining samples are used as the test sets. The average experimental results achieved by various methods are shown in Tables \ref{tab:COIL20} and \ref{tab:Caltech101}. As can be seen in Tables \ref{tab:COIL20} and \ref{tab:Caltech101}, our approach achieves competitive performance compared to other approaches. In particular, with the exception of MMC, our method performs significantly better on the Caltech101 database than any other method.
\begin{table}[htbp]
  \centering
  \caption{Classification performance (mean$\pm$std \%) of various approaches on the COIL20 database}
    \begin{tabular}{|c|c|c|c|c|}
    \hline
    Alg. & \# 15 & \# 20 & \# 25 & \# 30 \\
    \hline
    \hline
    RR    & 94.50$\pm$1.15  & 96.68$\pm$0.76  & 97.59$\pm$0.72  & 98.20$\pm$0.56  \\
    LDA   & 87.28$\pm$0.86  & 89.28$\pm$0.92  & 89.88$\pm$0.79  & 90.49$\pm$1.03  \\
    MMC   & 96.73$\pm$1.01  & 98.11$\pm$0.50  & 98.78$\pm$0.45  & 99.25$\pm$0.42  \\
    LSDA  & 87.89$\pm$0.96  & 89.53$\pm$1.04  & 90.36$\pm$0.95  & 90.62$\pm$0.75  \\
    LFDA  & 96.39$\pm$0.76  & 97.62$\pm$0.28  & 98.93$\pm$0.42  & 99.35$\pm$0.45  \\
    NMMP  & 92.57$\pm$1.37  & 93.91$\pm$1.06  & 94.52$\pm$0.49  & 94.90$\pm$0.61  \\
    SDA   & 92.24$\pm$1.34  & 95.16$\pm$1.08  & 96.21$\pm$0.78  & 97.60$\pm$0.63  \\
    SULDA & 85.61$\pm$1.10  & 87.66$\pm$1.20  & 88.60$\pm$1.06  & 88.98$\pm$0.96  \\
    L21SDA & 94.27$\pm$0.84  & 96.13$\pm$0.94  & 96.86$\pm$0.68  & 97.74$\pm$0.71  \\
    ALDE  & 96.61$\pm$0.99  & 98.11$\pm$0.67  & 98.70$\pm$0.37  & 99.21$\pm$0.49  \\
    ReLSR & 94.87$\pm$0.72  & 96.77$\pm$0.69  & 97.60$\pm$0.78  & 98.18$\pm$0.50  \\
    MPDA  & 92.57$\pm$1.37  & 93.91$\pm$1.06  & 94.52$\pm$0.49  & 94.90$\pm$0.61  \\
    RSLDA & 91.77$\pm$1.51  & 93.73$\pm$1.20  & 94.33$\pm$1.05  & 95.08$\pm$0.81  \\
    ALPR  & 95.32$\pm$0.69  & 97.03$\pm$0.71  & 97.93$\pm$0.58  & 98.35$\pm$0.58  \\
    \hline
    RLAR  & \textbf{96.92$\pm$0.87}  & \textbf{98.58$\pm$0.30}  & \textbf{99.10$\pm$0.41}  & \textbf{99.38$\pm$0.33}  \\
    \hline
    \end{tabular}%
  \label{tab:COIL20}%
\end{table}%
\begin{table}[htbp]
  \centering
  \caption{Classification performance (mean$\pm$std \%) of various approaches on the Caltech101 database}
    \begin{tabular}{|c|c|c|c|c|}
    \hline
    Alg. & \# 10 & \# 15 & \# 20 & \# 25 \\
    \hline
    \hline
    RR    & 58.23$\pm$0.84  & 60.50$\pm$0.49  & 61.25$\pm$0.53  & 62.03$\pm$0.49  \\
    LDA   & 56.64$\pm$0.94  & 55.73$\pm$0.52  & 50.26$\pm$0.73  & 34.72$\pm$0.84  \\
    MMC   & 61.95$\pm$0.88  & 66.45$\pm$0.62  & 69.39$\pm$0.57  & 71.72$\pm$0.42  \\
    LSDA  & 56.56$\pm$0.96  & 55.80$\pm$0.41  & 50.33$\pm$0.64  & 34.69$\pm$0.82  \\
    LFDA  & 55.67$\pm$1.25  & 61.05$\pm$0.74  & 64.82$\pm$0.69  & 67.66$\pm$0.87  \\
    NMMP  & 54.25$\pm$0.96  & 54.46$\pm$0.65  & 51.52$\pm$0.80  & 42.92$\pm$0.78  \\
    SDA   & 56.76$\pm$0.75  & 61.50$\pm$0.83  & 64.48$\pm$0.72  & 66.93$\pm$0.37  \\
    SULDA & 46.75$\pm$0.84  & 48.17$\pm$0.84  & 44.73$\pm$0.70  & 30.96$\pm$0.61  \\
    L21SDA & 54.41$\pm$1.20  & 56.54$\pm$0.88  & 54.96$\pm$0.70  & 52.66$\pm$0.41  \\
    ALDE  & 58.44$\pm$0.79  & 62.57$\pm$0.68  & 65.28$\pm$0.50  & 67.53$\pm$0.49  \\
    ReLSR & 61.68$\pm$1.03  & 65.25$\pm$0.58  & 67.32$\pm$0.61  & 69.01$\pm$0.51  \\
    MPDA  & 54.23$\pm$1.09  & 54.41$\pm$0.53  & 51.52$\pm$0.79  & 42.97$\pm$0.81  \\
    RSLDA & 51.82$\pm$1.20  & 51.11$\pm$0.77  & 46.11$\pm$0.58  & 36.22$\pm$0.79  \\
    ALPR  & 61.50$\pm$1.07  & 65.11$\pm$0.46  & 66.45$\pm$0.66  & 67.09$\pm$0.44  \\
    \hline
    RLAR  & \textbf{62.81$\pm$1.03}  & \textbf{67.64$\pm$0.68}  & \textbf{70.52$\pm$0.58}  & \textbf{72.61$\pm$0.41}  \\
    \hline
    \end{tabular}%
  \label{tab:Caltech101}%
\end{table}%
\subsection{Robustness Evaluation}
To investigate the sensitivity of our RLAR to outliers, we conduct two groups of comparative experiments involving the intensity and quantity of outliers on the AR database. The subset of AR database containing 1400 face images mentioned in Section \ref{face:recognition} is denoted as 'Subset1' here. And, we extract another subset with natural occlusion from the AR database, named as 'Subset2', including 600 images blocked by glasses and 600 images blocked by scarves from 50 male and 50 female subjects. Besides, we artificially block facial images in 'Subset1' by a 'baboon' image with varying block sizes. Some sample images of natural occlusion and artificial occlusion are illustrated in Fig. \ref{fig:AR:Samples} and Fig. \ref{fig:AR:CorruptedSamples}, respectively. The detailed experimental schemes and results are as follows.
\begin{figure}[htpb]
  \centering
  \includegraphics[scale=0.5]{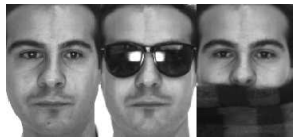}\\
  \caption{One image without blocking and two images with glasses and scarf blocking selected from the AR database.}\label{fig:AR:Samples}
\end{figure}
\begin{figure}[htpb]
  \centering
  \includegraphics[scale=0.5]{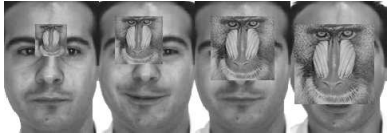}\\
  \caption{Sample images corrupted by a 'baboon' image with varying block sizes.}\label{fig:AR:CorruptedSamples}
\end{figure}

We first observe the variation of classification performance with the intensity of outliers. We randomly select 8 samples from each class of 'Subset1' to form the training set, among which 3 images are corrupted by a randomly located square block of a 'baboon' image, and the remaining samples were used as the test set. The block size determines the occlusion level of an image. Then, we evaluate the classification performance of various methods at four occlusion levels, and then list the average experimental results for 10 trials
independently in Table \ref{tab:artificial:occlusion}. For the sensitivity of quantitative outliers, we randomly select 1, 2, and 3 samples from the above 8 training samples, and then replace them with the glasses and scarf blocking images selected randomly in 'Subset2'. Similarly, we independently perform 10 trials for each evaluation, and present the average experimental results in Table \ref{tab:original:occlusion}.

In Tables \ref{tab:artificial:occlusion} and \ref{tab:original:occlusion}, as the intensity and quantity of outliers increase, we observe that the classification accuracies achieved by various methods gradually decreases. It is worth noting that SULDA, who performs well in the above experiments, fails completely in the face of outlier interference. And in Table \ref{tab:artificial:occlusion}, NMMP also fails at the occlusion level of $30\times30$. It can be concluded from Table \ref{tab:artificial:occlusion} that our method is superior to all methods in terms of recognition accuracy at different occlusion levels. In Table \ref{tab:original:occlusion}, the proposed method is slightly inferior to ALPR except when the number of scarf blocking images is 2, while it outperforms all of the compared methods in other cases. In general, with the increase of the intensity and quantity of outliers, the recognition accuracies of our RLAR does not decrease significantly, which also demonstrates that our method has excellent performance in resisting outliers.
\begin{table}[htbp]
  \centering
  \caption{Classification performance (\%) on the AR database with varying block sizes.}
    \begin{tabular}{|c|c|c|c|c|}
    \hline
    \multirow{2}[4]{*}{Alg.} & \multicolumn{4}{c|}{Occlusion level} \\
\cline{2-5}     & 15$\times$15    & 20$\times$20    & 25$\times$25   & 30$\times$30\\
    \hline
    \hline
    RR    & 96.60  & 96.08  & 95.83  & 94.50  \\
    LDA   & 84.82  & 83.72  & 78.13  & 56.32  \\
    MMC   & 93.73  & 92.60  & 91.63  & 91.08  \\
    LSDA  & 85.22  & 83.88  & 78.27  & 56.23  \\
    LFDA  & 84.22  & 81.98  & 79.88  & 77.80  \\
    NMMP  & 92.27  & 91.55  & 88.65  & 2.40  \\
    SDA   & 96.57  & 95.82  & 95.72  & 93.90  \\
    SULDA & 30.73  & 31.53  & 27.87  & 8.22  \\
    L21SDA & 97.25  & 96.53  & 96.50  & 94.22  \\
    ALDE  & 93.73  & 92.18  & 90.63  & 88.60  \\
    ReLSR & 96.55  & 95.72  & 95.40  & 93.53  \\
    MPDA  & 92.85  & 91.98  & 91.80  & 91.52  \\
    RSLDA & 88.98  & 87.12  & 82.50  & 66.53  \\
    ALPR  & 98.25  & 97.80  & 97.43  & 94.90  \\
    RLAR  & \textbf{98.43 } & \textbf{98.02 } & \textbf{97.92 } & \textbf{95.23 } \\
    \hline
    \end{tabular}%
  \label{tab:artificial:occlusion}%
\end{table}%
\begin{table}[htbp]
  \centering
  \caption{Classification performance (\%) on the AR database with diverse number of glasses and scarves blocking images.}
    \begin{tabular}{|c|c|c|c|c|c|c|}
    \hline
    \multirow{2}[4]{*}{Alg.} & \multicolumn{3}{c|}{Glasses occlusion} & \multicolumn{3}{c|}{Scarf occlusion} \\
\cline{2-7}          & \# 1     & \# 2     & \# 3     & \# 1     & \# 2     & \# 3 \\
    \hline
    \hline
    RR    & 97.60  & 97.13  & 96.45  & 97.57  & 97.05  & 96.42  \\
    LDA   & 87.13  & 86.48  & 85.28  & 87.47  & 86.20  & 84.13  \\
    MMC   & 95.35  & 94.47  & 93.20  & 95.98  & 95.63  & 94.98  \\
    LSDA  & 87.02  & 86.65  & 84.90  & 87.87  & 86.32  & 84.50  \\
    LFDA  & 87.37  & 85.08  & 82.05  & 88.57  & 87.43  & 85.42  \\
    NMMP  & 93.57  & 92.70  & 91.92  & 93.30  & 92.57  & 91.47  \\
    SDA   & 97.70  & 97.15  & 96.40  & 97.80  & 97.17  & 96.73  \\
    SULDA & 41.50  & 37.72  & 29.48  & 40.72  & 36.98  & 28.40  \\
    L21SDA & 97.90  & 97.55  & 96.90  & 98.02  & 97.73  & 97.32  \\
    ALDE  & 95.47  & 94.18  & 92.02  & 95.55  & 93.72  & 90.77  \\
    ReLSR & 97.58  & 96.87  & 96.10  & 97.60  & 96.98  & 96.07  \\
    MPDA  & 94.80  & 93.65  & 92.38  & 95.00  & 93.95  & 92.12  \\
    RSLDA & 91.92  & 90.67  & 87.67  & 92.83  & 91.82  & 90.78  \\
    ALPR  & 98.55  & 98.10  & 97.73  & 98.63  & \textbf{98.33} & 97.93  \\
    RLAR  & \textbf{98.67} & \textbf{98.38} & \textbf{98.05} & \textbf{98.72} & 98.32  & \textbf{97.98} \\
    \hline
    \end{tabular}%
  \label{tab:original:occlusion}%
\end{table}%

\subsection{Parameter Sensitivity Analysis}
In this section, we examine the parameter sensitivity of the proposed RLAR, which involves three hyper-parameters, namely the number of neighbors $K$, the regularization parameter $\alpha$, and the tradeoff coefficient $\beta$. In all of the experiments above, we set $K$ to either a fixed 3 or a fixed 7, which we mentioned in our experimental settings. Here we mainly focus on discussing the impact of changes in $\alpha$ and $\beta$ on the performance of the proposed model. The predetermined adjustment coordinate set of these two parameters is set as $\{0.001,0.005,0.01,0.05,0.1,0.5,1 ,10,100,1000\}$. The recognition results versus two parameters on 10 databases are visualized in Fig. \ref{fig:Parameter:sensitivity}, where the value of horizontal and vertical coordinates represents the subscript of the coordinate set, and the number or proportion of training samples per class is indicated in parenthesis of the corresponding caption. It can be observed that the two parameters are not allowed to be too large and not allowed to be too small, and generalized classification performance can be achieved near $[0.1, 0.1]$.
\begin{figure*}
  \centering
  \subfigure[Dermatology (20\%)]{
  \centering
  \includegraphics[scale=0.22]{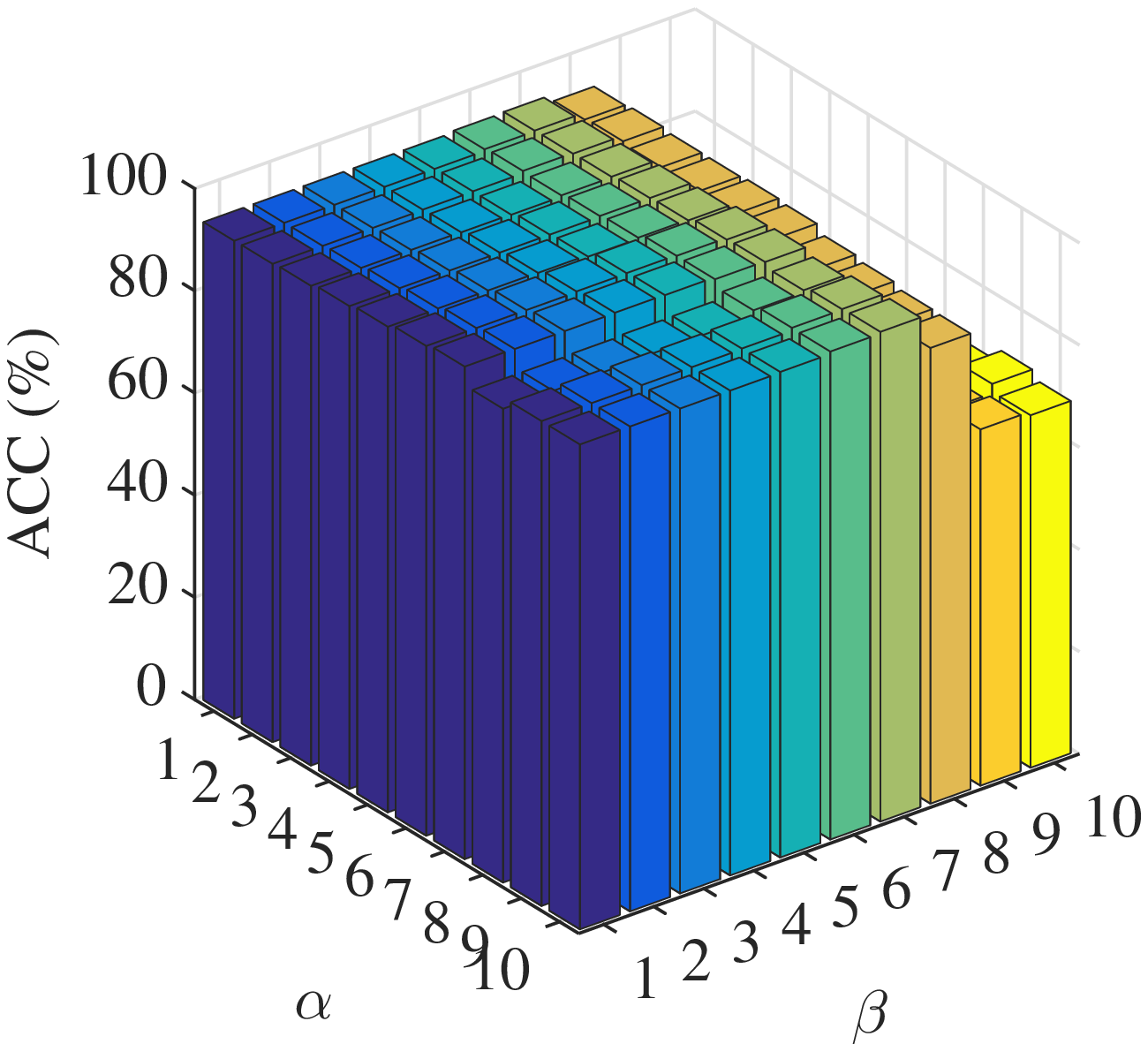}}
  \subfigure[Diabetes (20\%)]{
  \centering
  \includegraphics[scale=0.22]{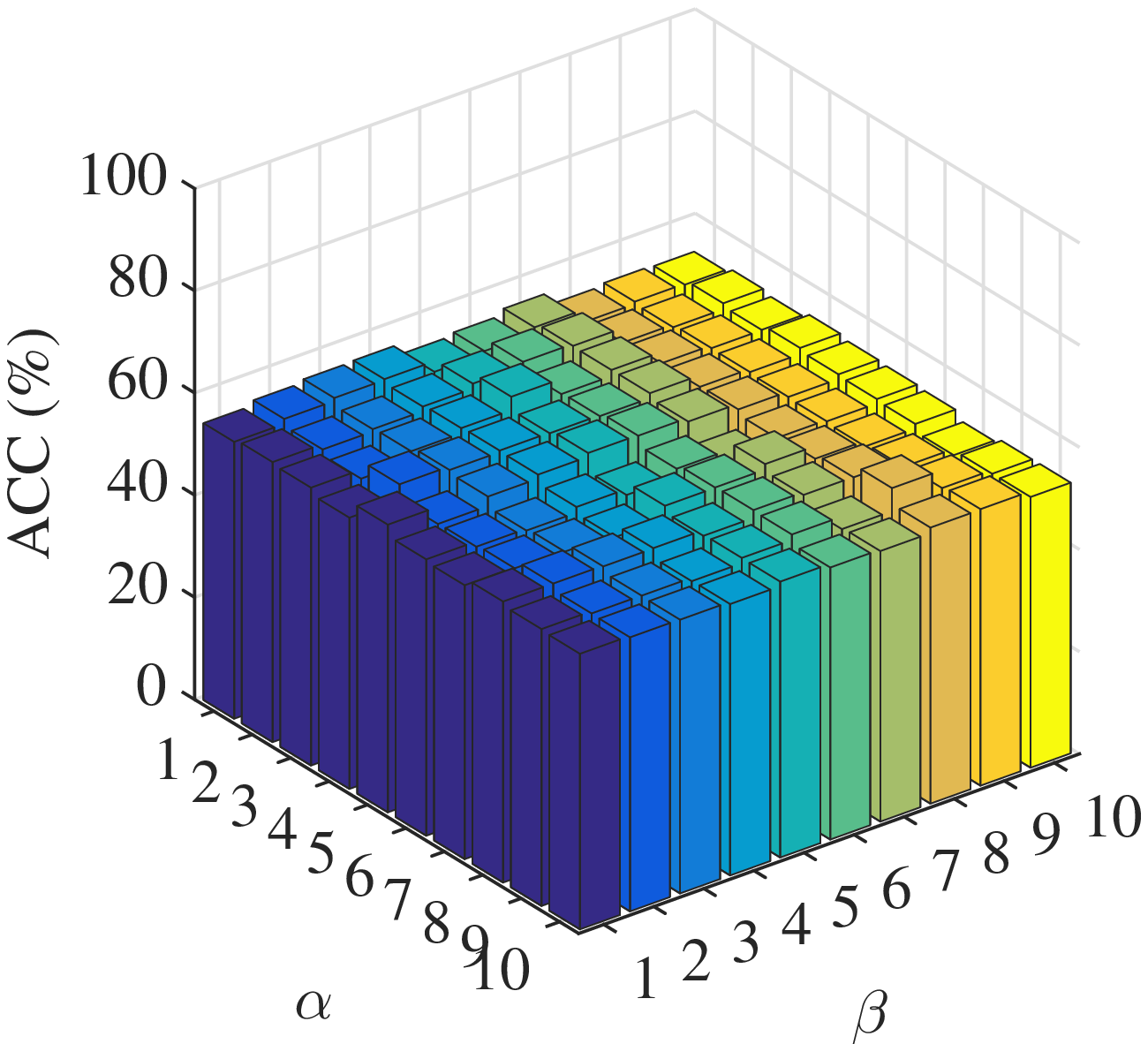}}
  \subfigure[Ionosphere (20\%)]{
  \centering
  \includegraphics[scale=0.22]{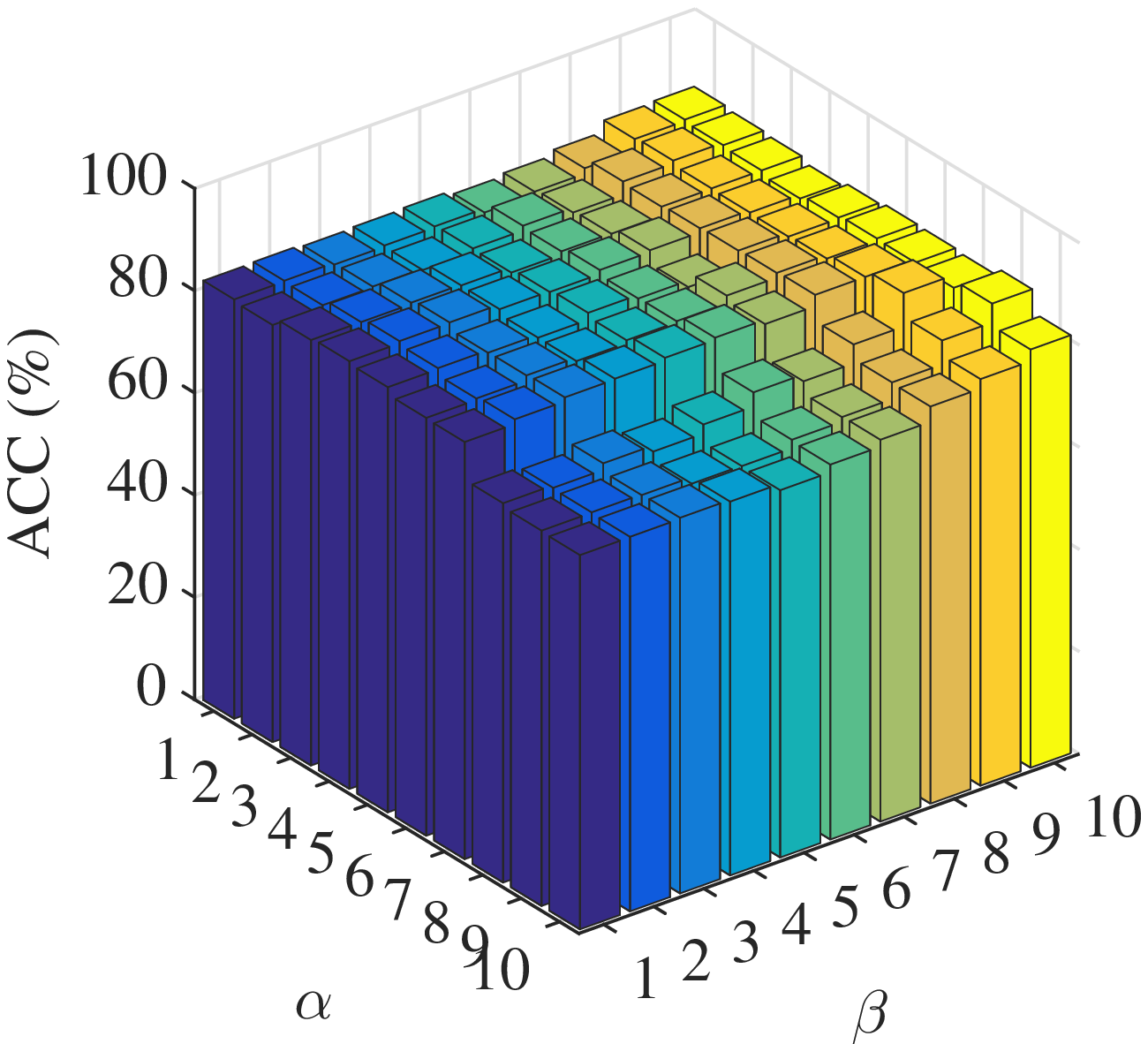}}
  \subfigure[Iris (20\%)]{
  \centering
  \includegraphics[scale=0.22]{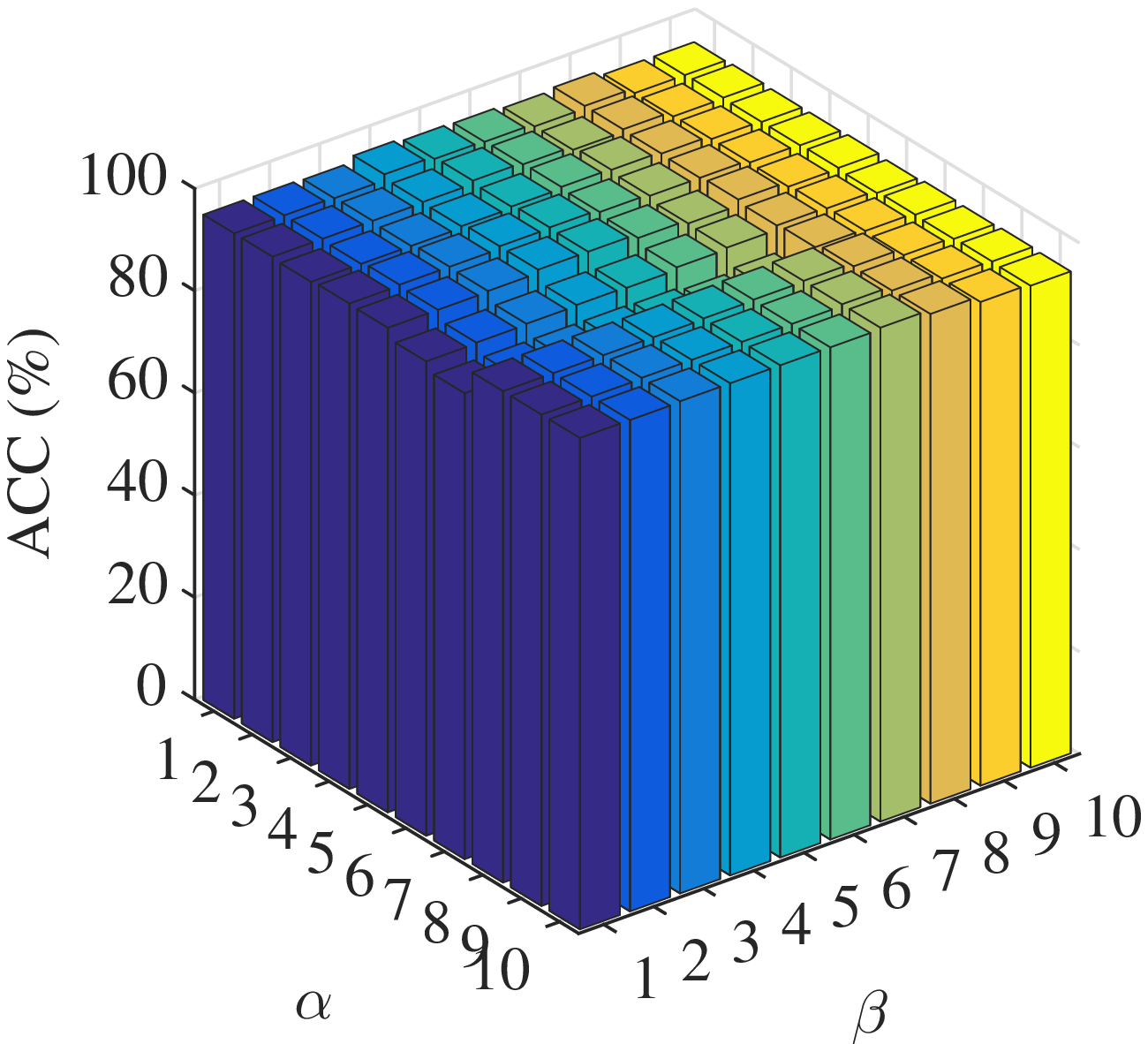}}
  \subfigure[Wine (20\%)]{
  \centering
  \includegraphics[scale=0.22]{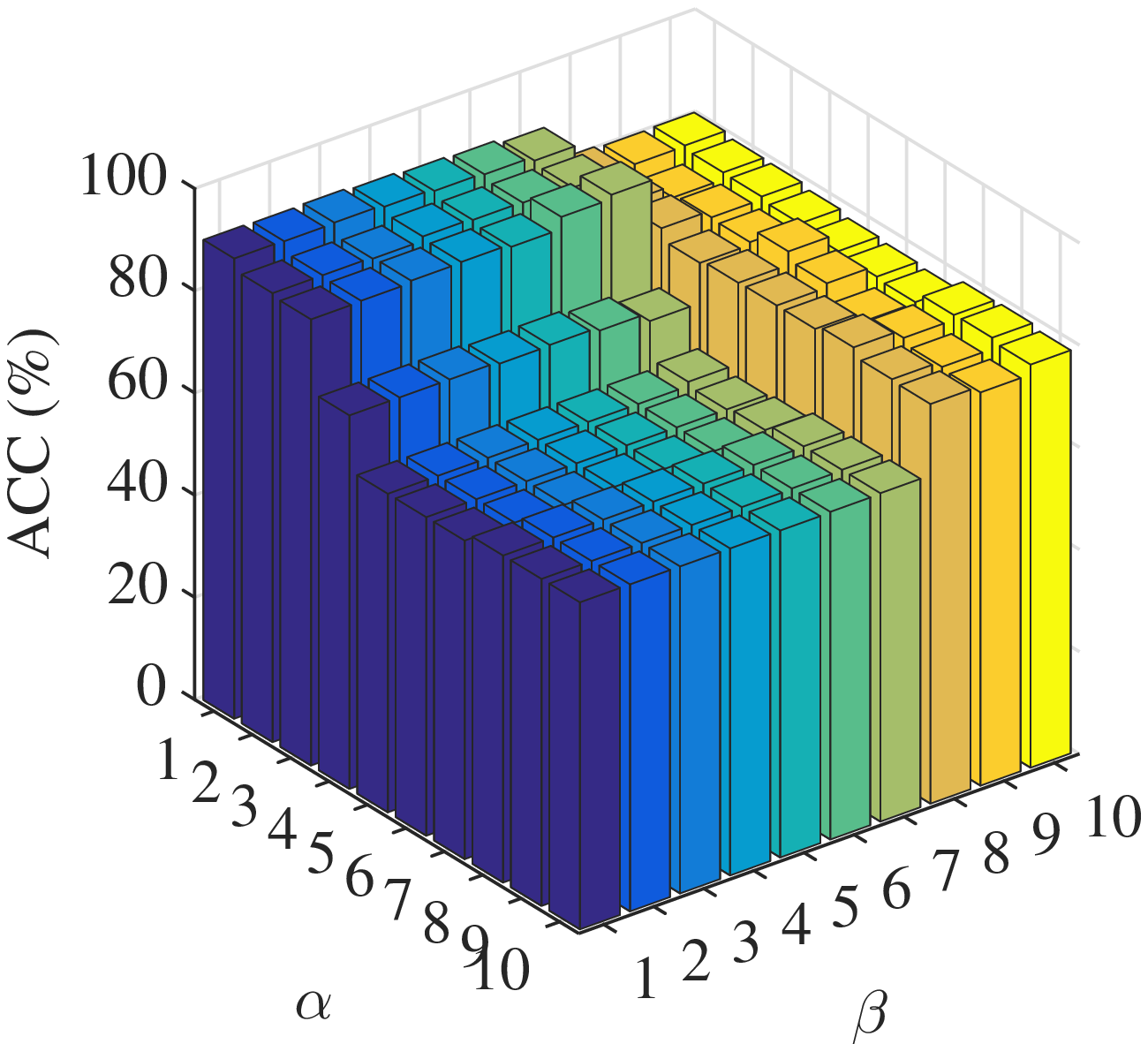}}

  \centering
  \subfigure[Binalpha (\# 19)]{
  \centering
  \includegraphics[scale=0.22]{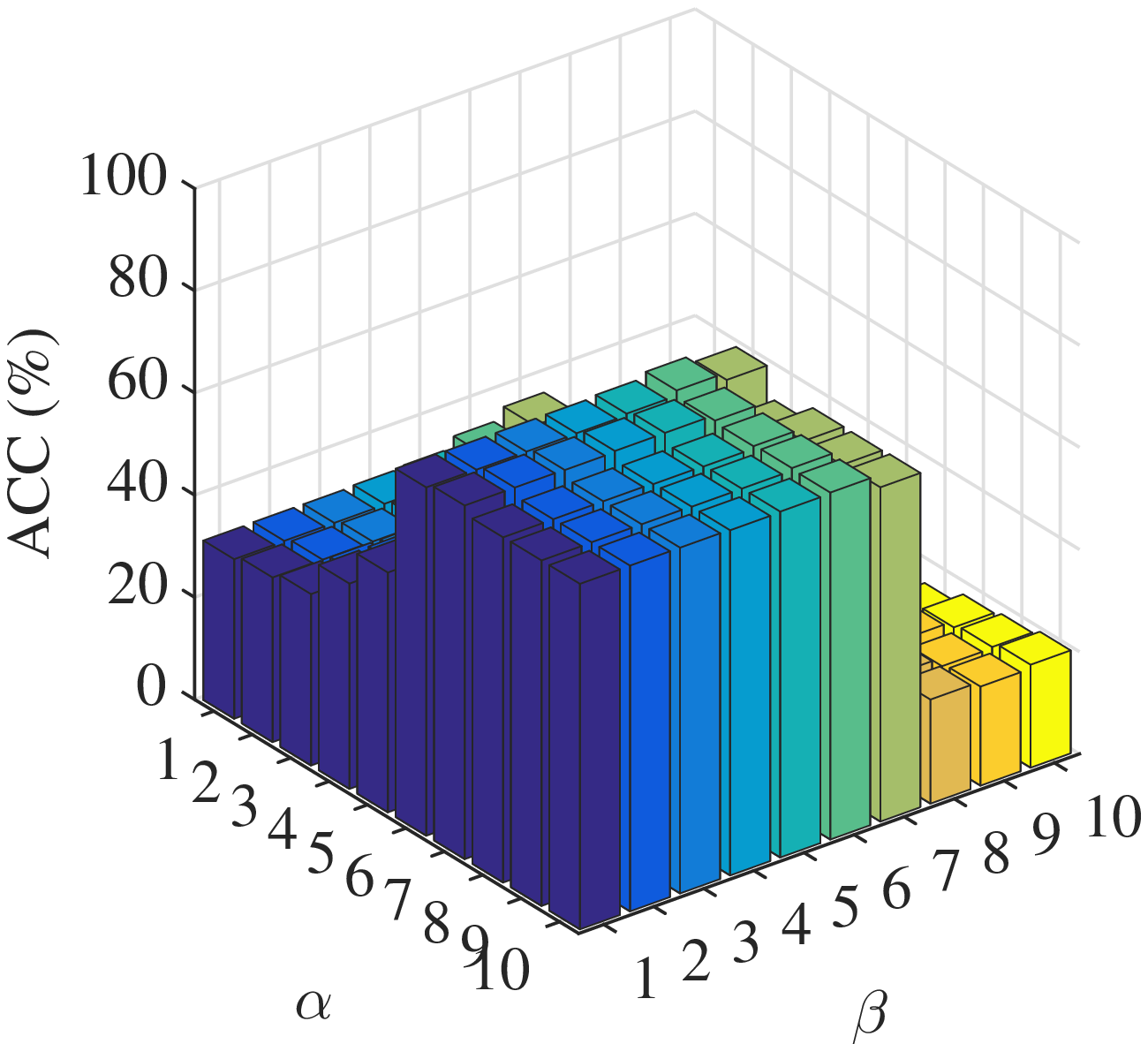}}
  \subfigure[YaleB (\# 30)]{
  \centering
  \includegraphics[scale=0.22]{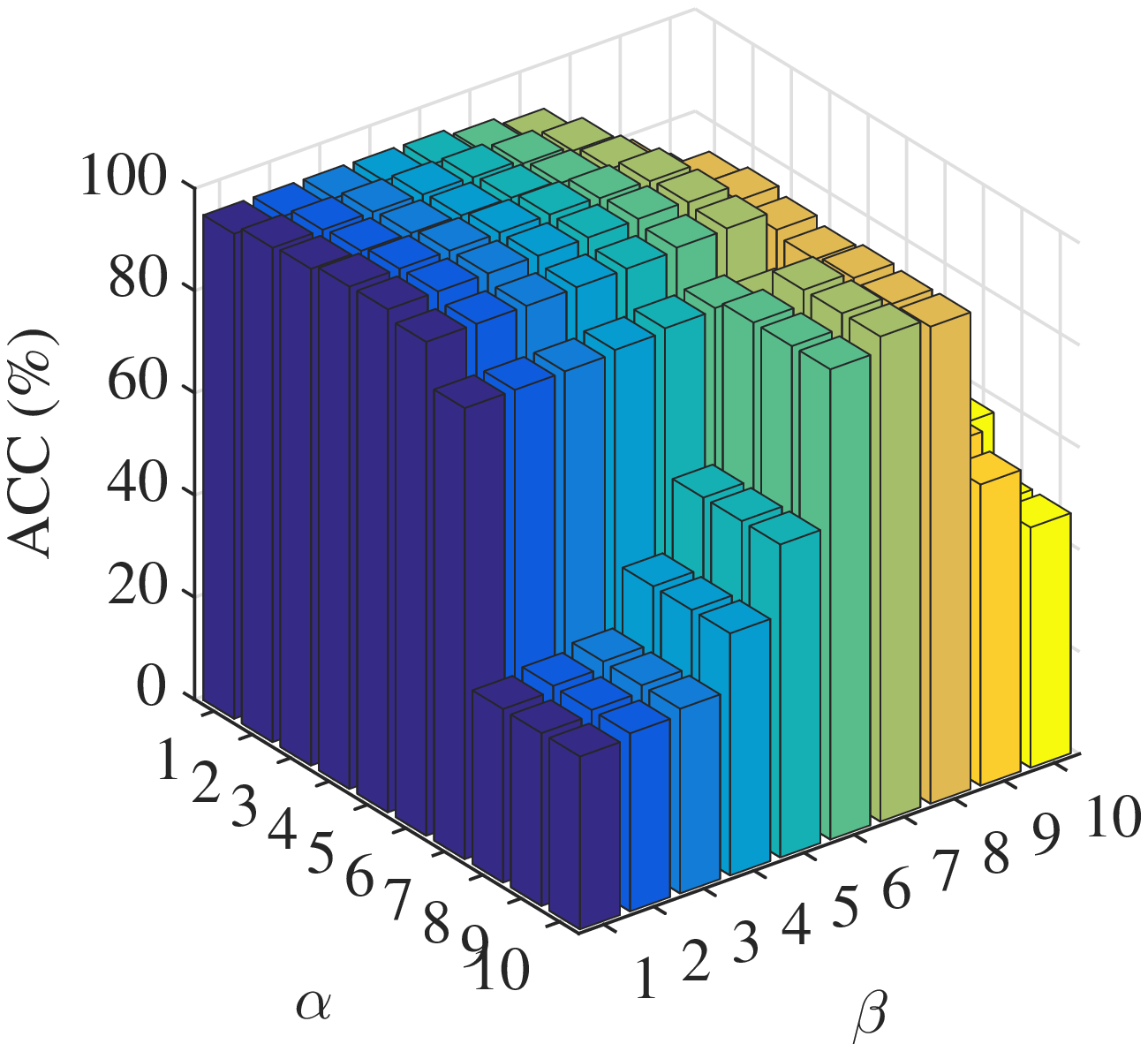}}
  \subfigure[AR (\# 8)]{
  \centering
  \includegraphics[scale=0.22]{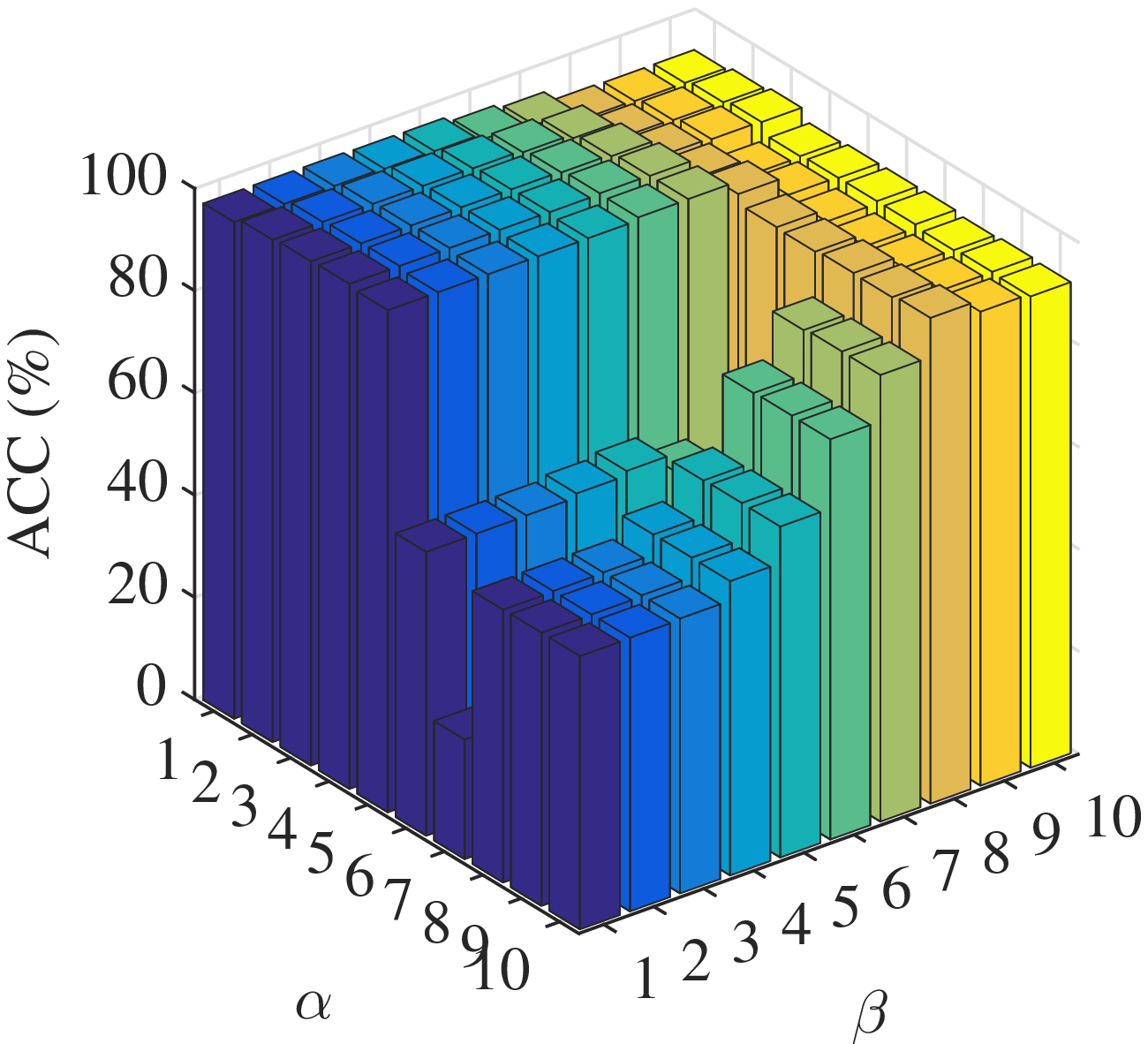}}
  \subfigure[COIL20 (\# 30)]{
  \centering
  \includegraphics[scale=0.22]{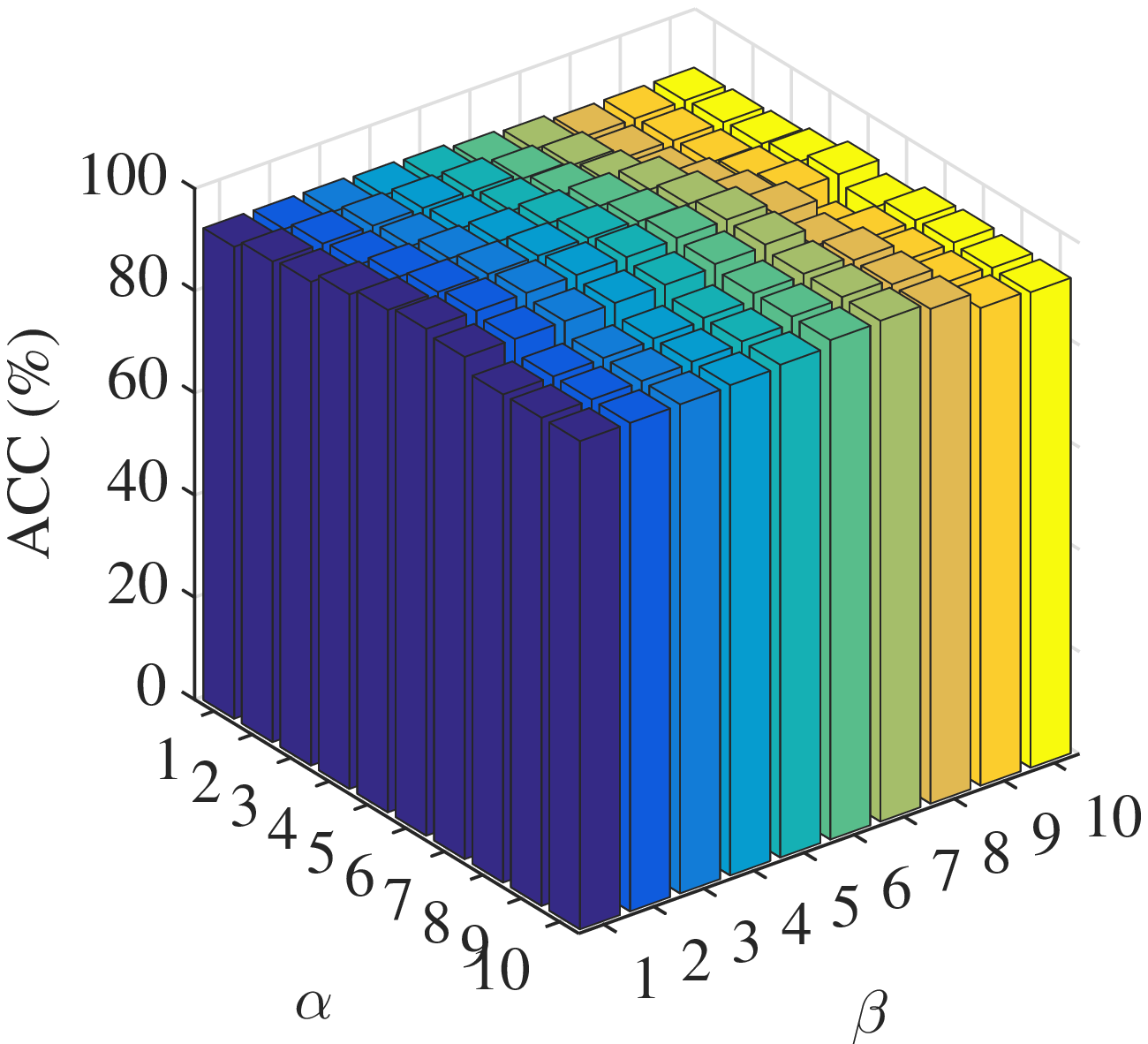}}
  \subfigure[Caltech101 (\# 25)]{
  \centering
  \includegraphics[scale=0.22]{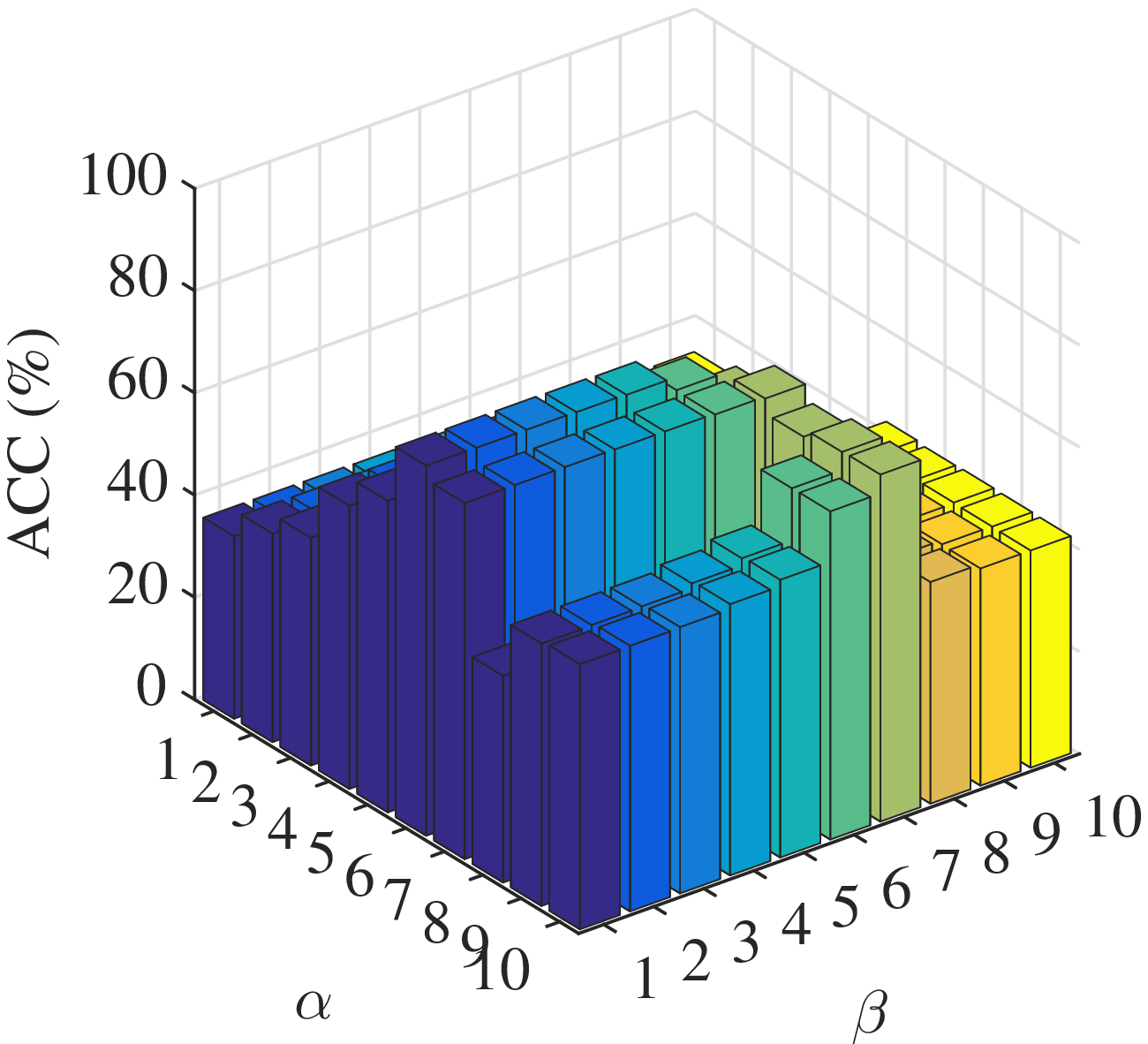}}
  \caption{Classification performance evaluation (\%) of the proposed RLAR versus hyper-parameters $\alpha$ and $\beta$ on ten different databases. }\label{fig:Parameter:sensitivity}
\end{figure*}

\begin{figure}[htpb]
  \centering
  \includegraphics[scale=0.4]{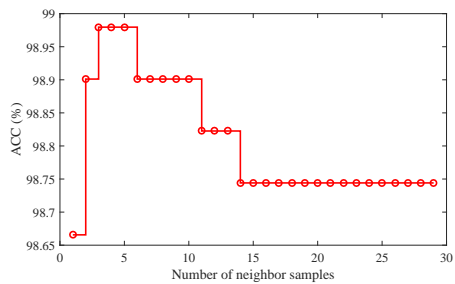}\\
  \caption{Classification performance evaluation (\%) of the proposed RLAR versus the number of neighbor samples $K$ on the YaleB database.}\label{fig:YaleB:K:Selection}
\end{figure}
Besides, we fix the optimal $\alpha$ and $\beta$ obtained through grid search, and observe the effect of the number of neighbors from 1 to 29 on the classification performance on the YaleB database. The recognition results are illustrated in the Fig. \ref{fig:YaleB:K:Selection}, from which it is observed that the classification performance varies slightly with $K$ and reaches the optimal at 3, 4, and 5 neighbors. Moreover, it is acceptable to set $K$ as a fixed value in all the previous experiments.

\subsection{Convergence Study}
The model we built involves multiple variables and is non-smooth, which inspires us to develop an iterative optimization strategy for solving it. The convergence of the optimization algorithm is theoretically guaranteed in Section \ref{algorithm_analysis}. Here, we experimentally verify the convergence performance of the proposed optimization algorithm on 10 databases. Accordingly, we show the convergence curves in Fig. \ref{fig:Convergence}, from which we observe that all convergence curves are indeed monotonically decreasing and tend to flatten within 30 iterations. The validity of the proposed RLAR is also confirmed by the mutual support between theoretical proof and experimental results.
\begin{figure*}
  \centering
  \subfigure[Dermatology (20\%)]{
  \centering
  \includegraphics[scale=0.22]{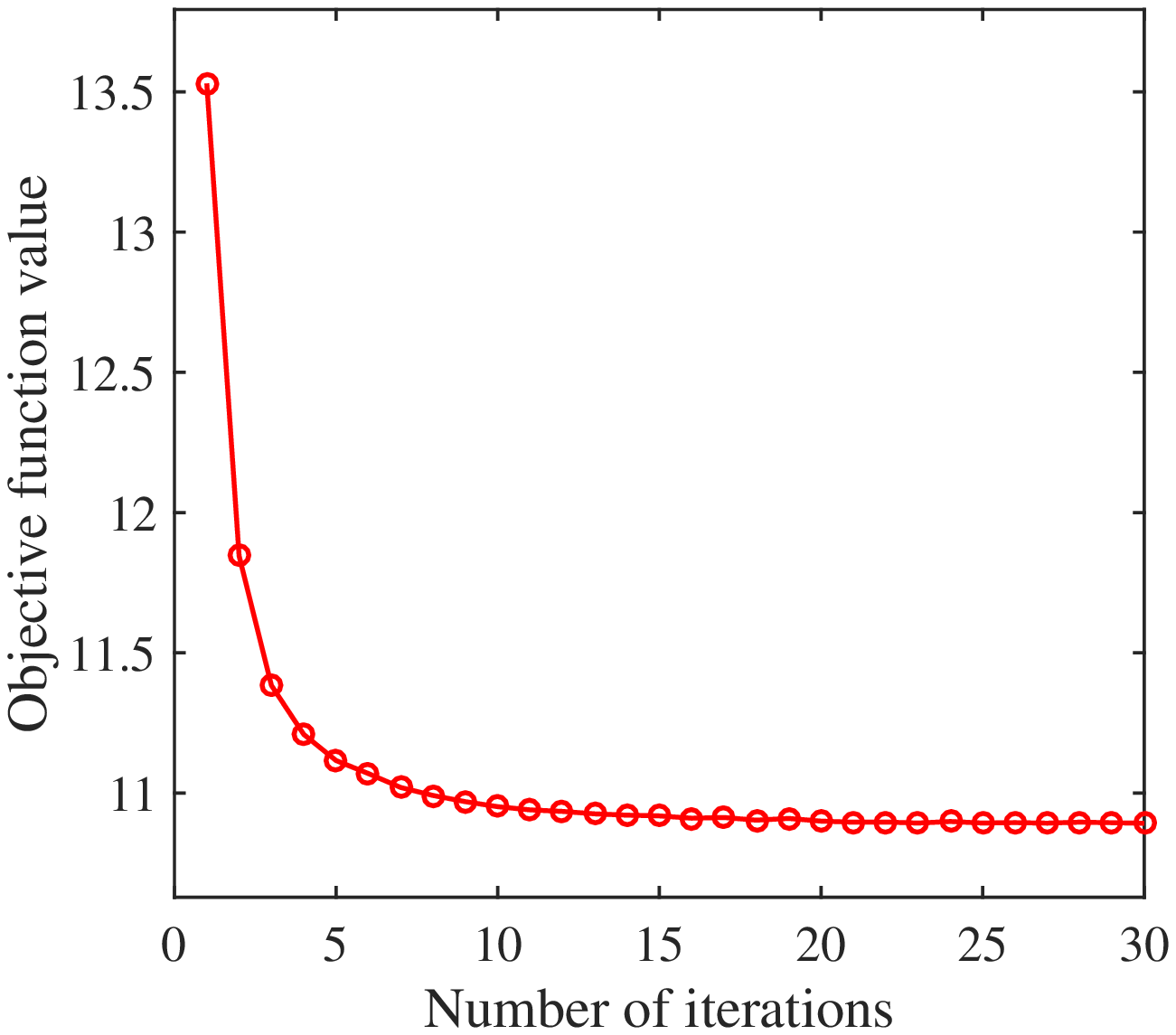}}
  \subfigure[Diabetes (20\%)]{
  \centering
  \includegraphics[scale=0.22]{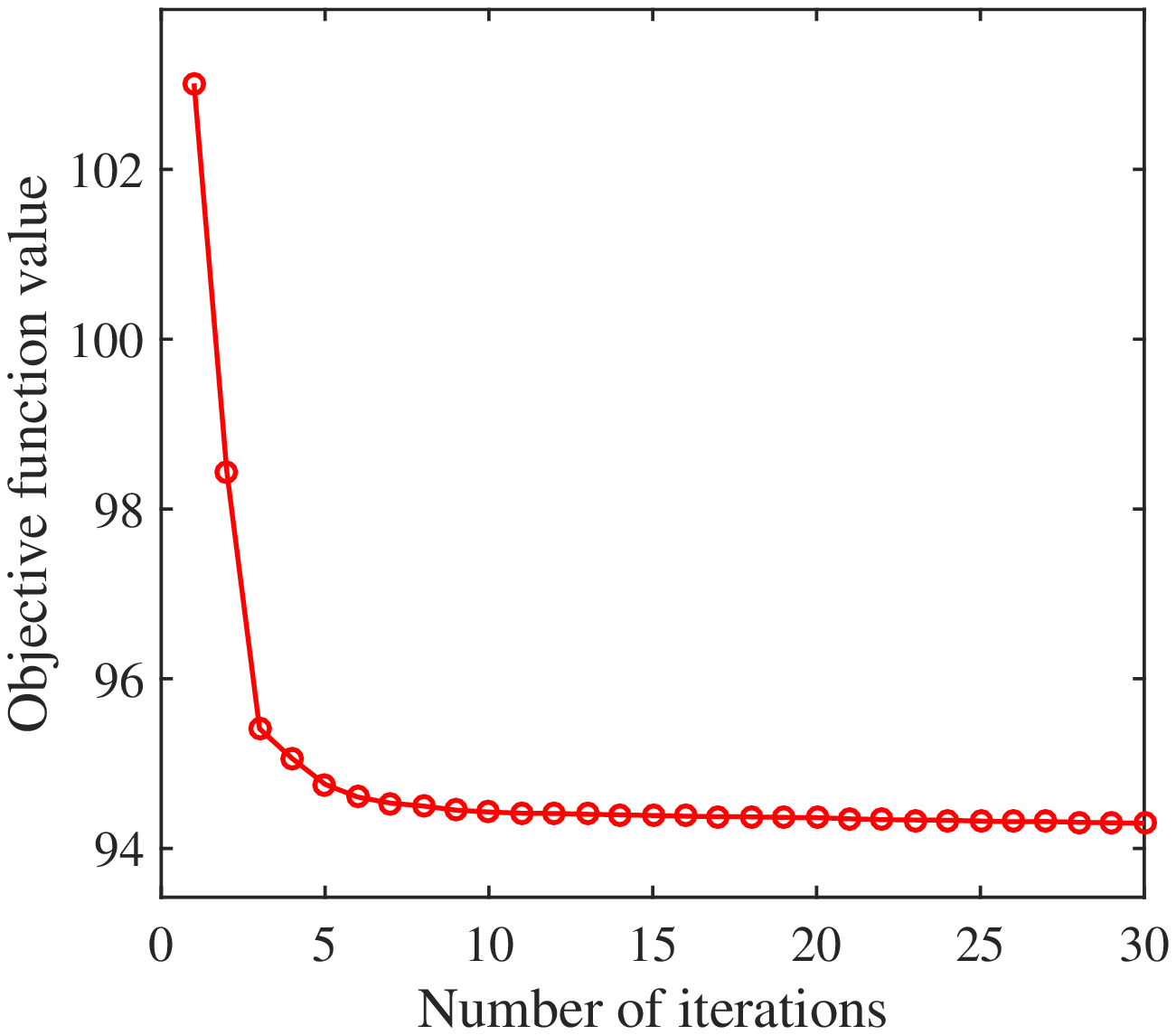}}
  \subfigure[Ionosphere (20\%)]{
  \centering
  \includegraphics[scale=0.22]{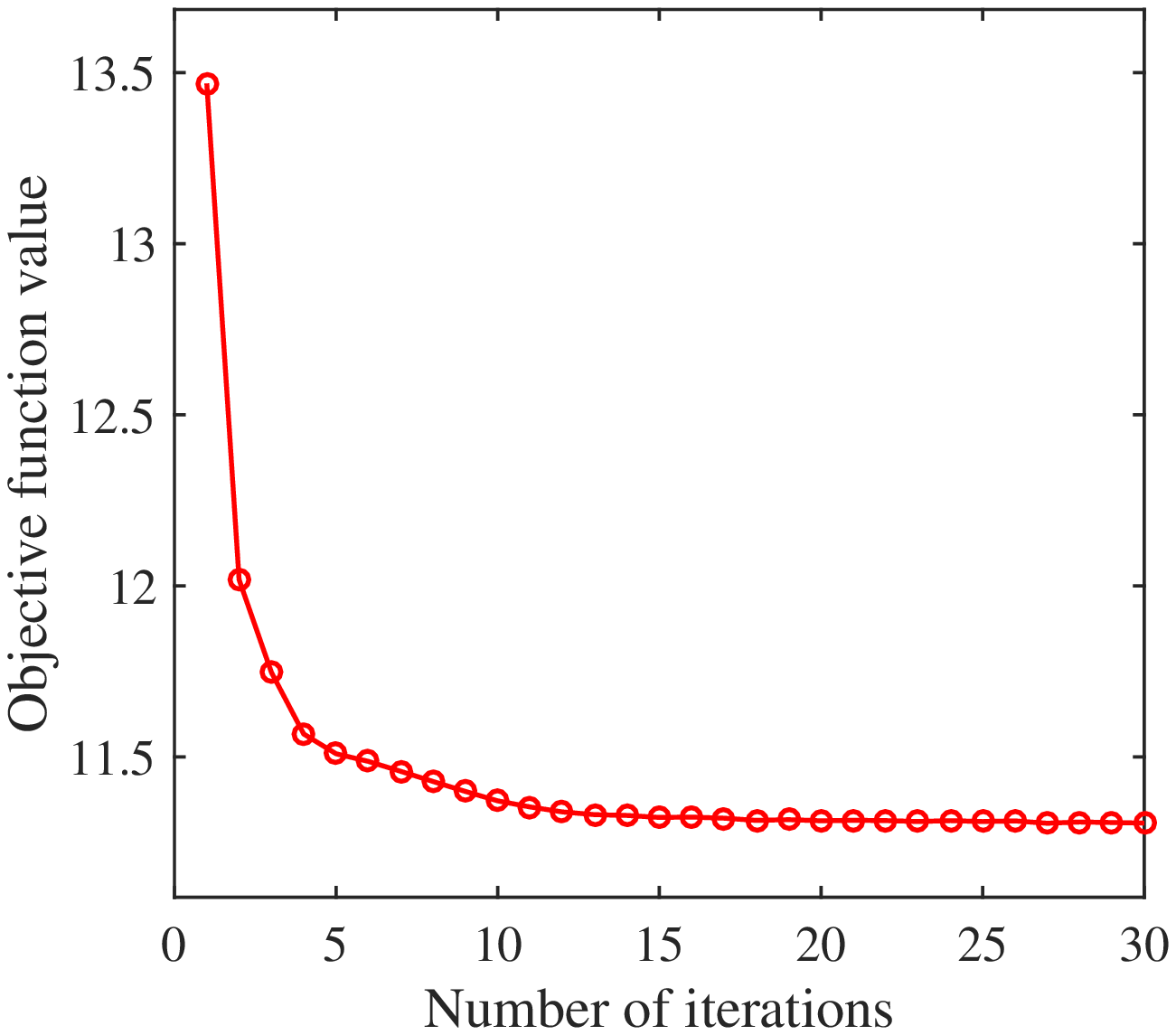}}
  \subfigure[Iris (20\%)]{
  \centering
  \includegraphics[scale=0.22]{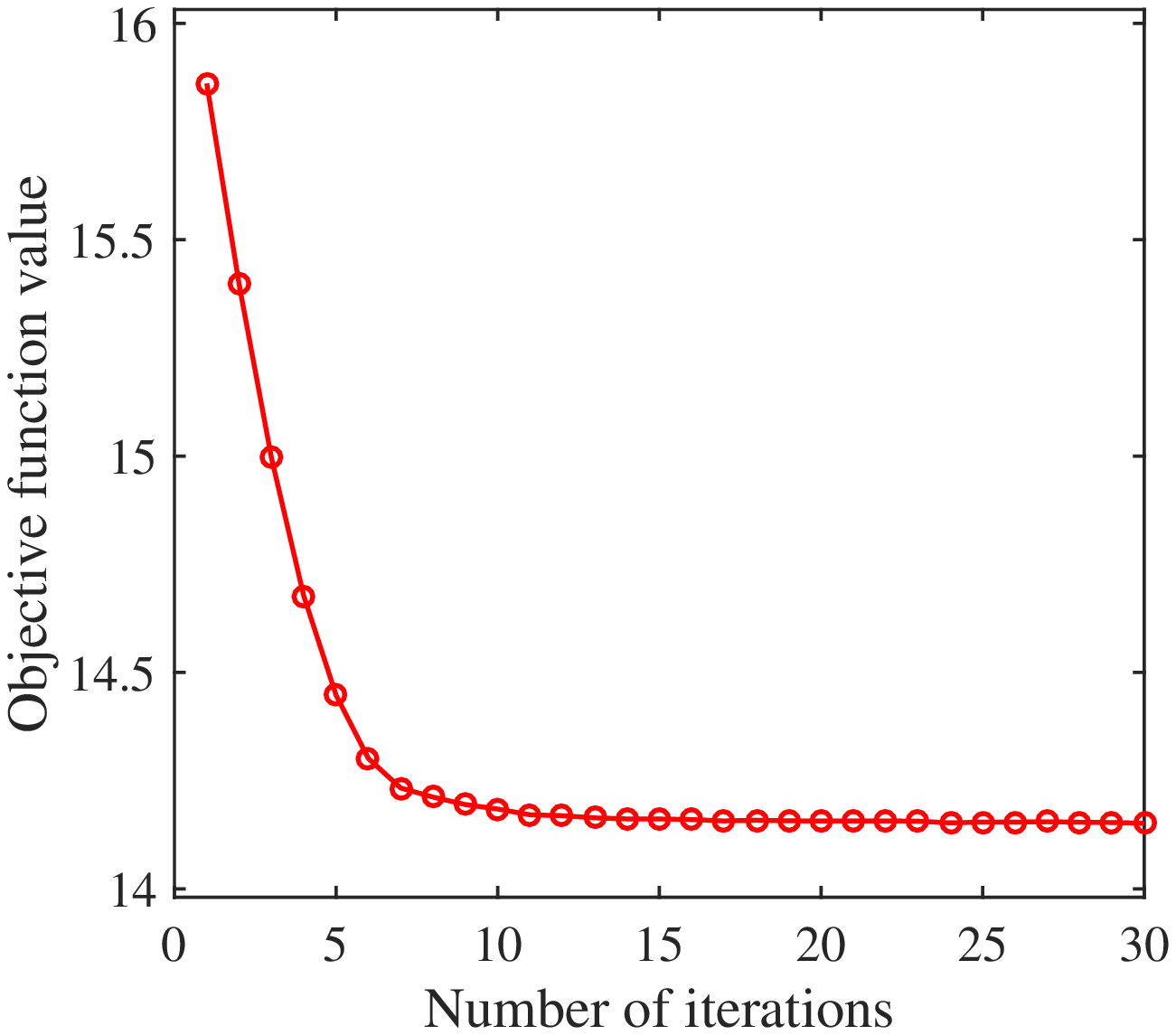}}
  \subfigure[Wine (20\%)]{
  \centering
  \includegraphics[scale=0.22]{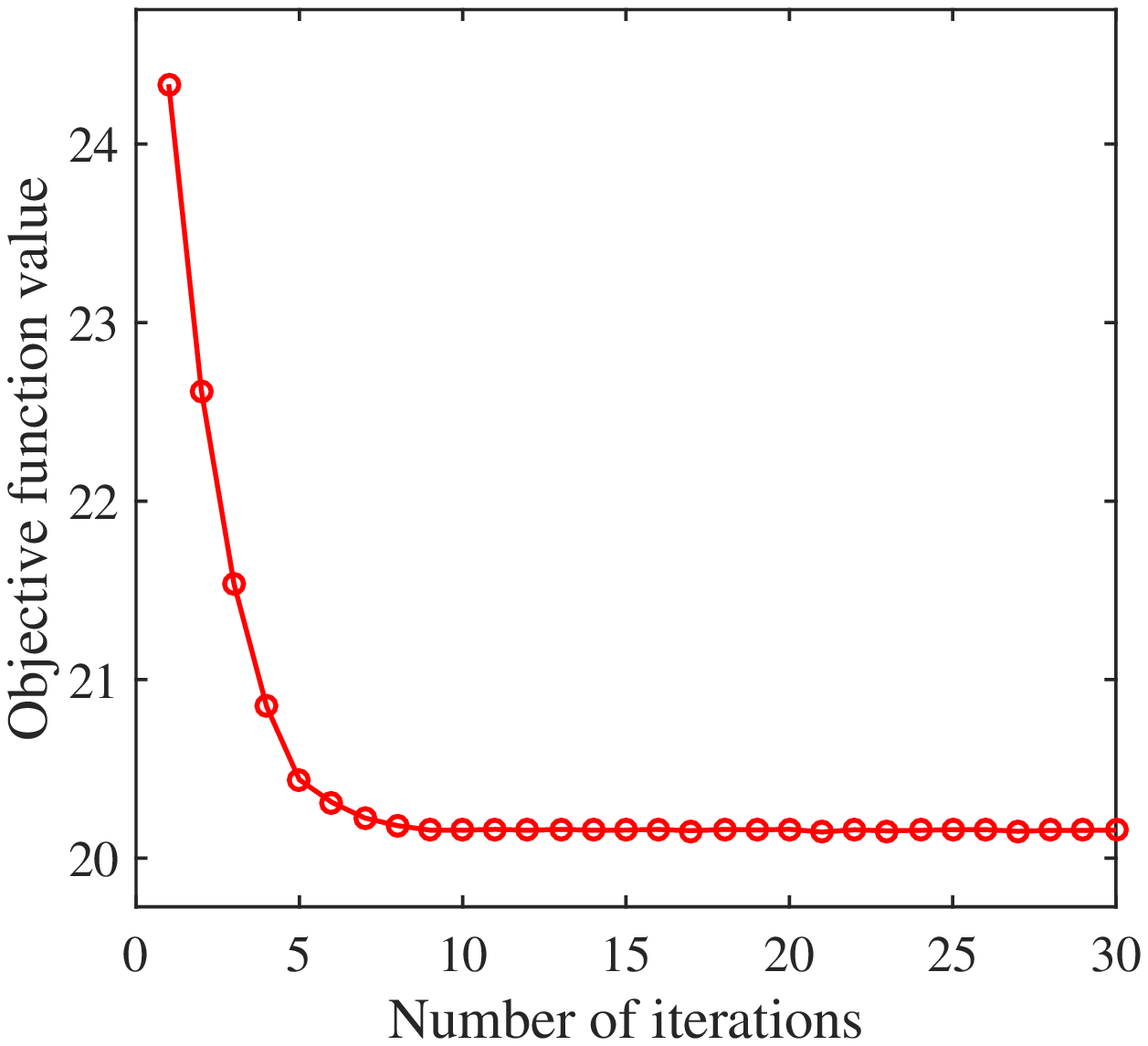}}

  \centering
  \subfigure[Binalpha (\# 19)]{
  \centering
  \includegraphics[scale=0.22]{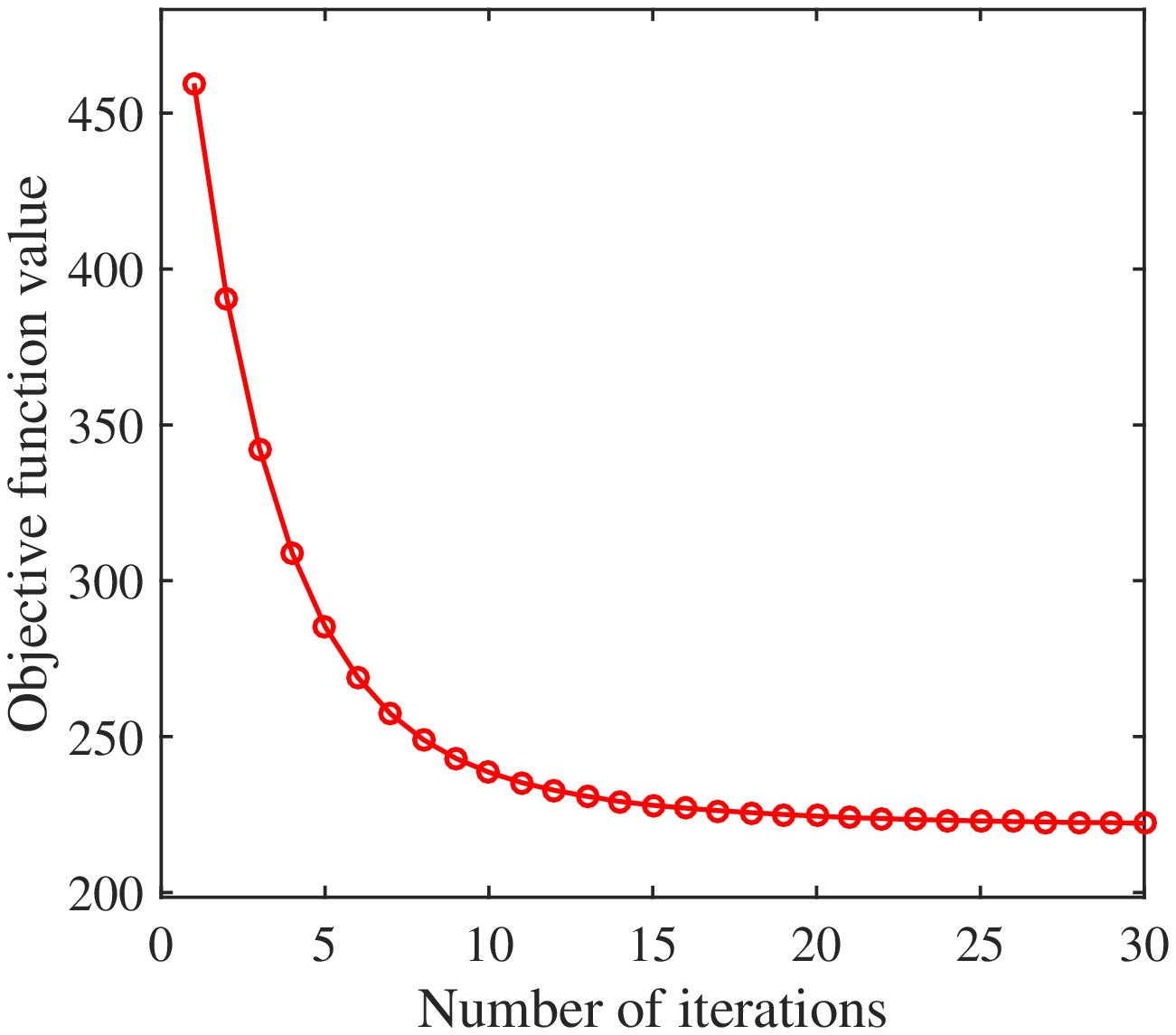}}
  \subfigure[YaleB (\# 30)]{
  \centering
  \includegraphics[scale=0.22]{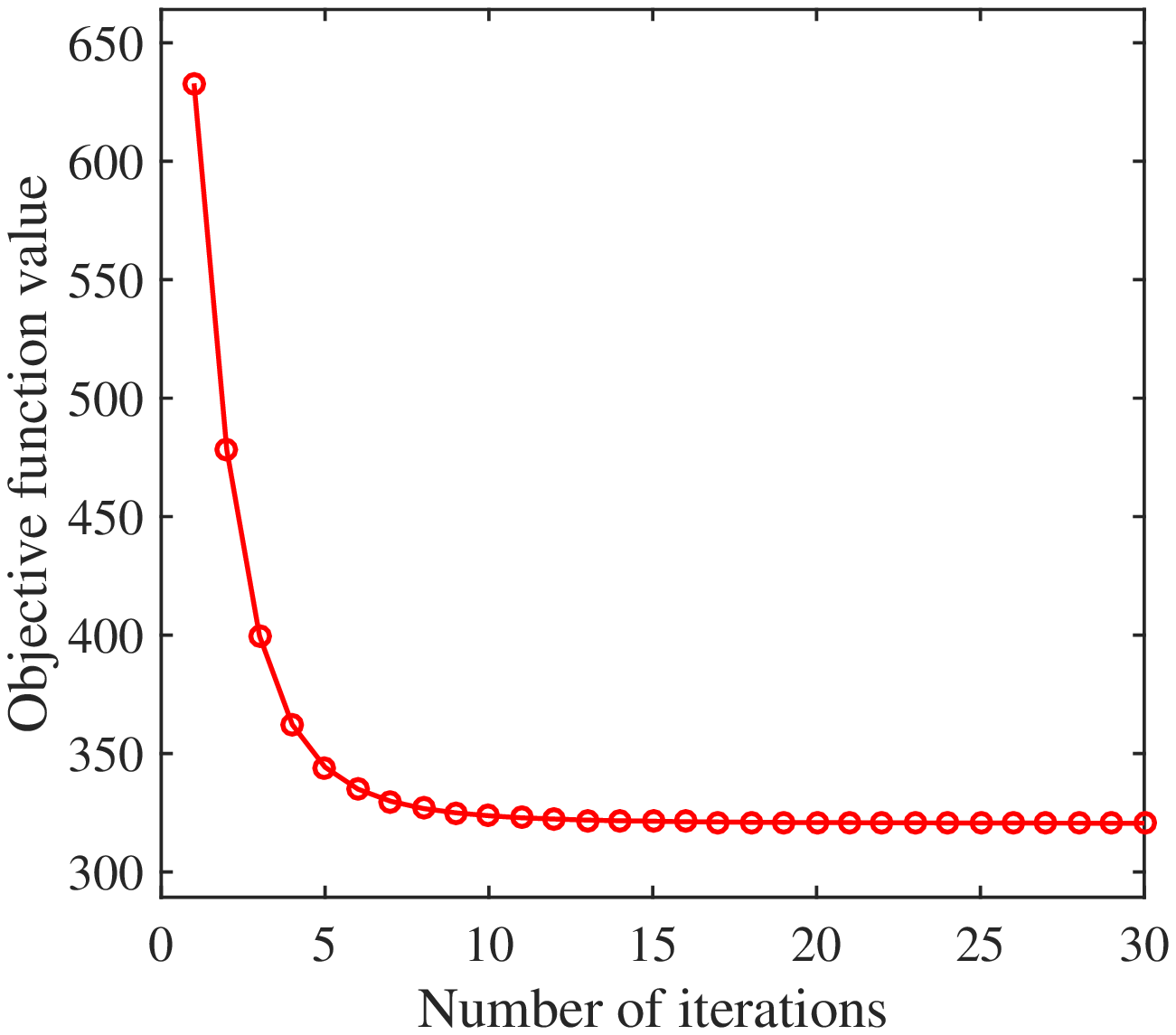}}
  \subfigure[AR (\# 8)]{
  \centering
  \includegraphics[scale=0.22]{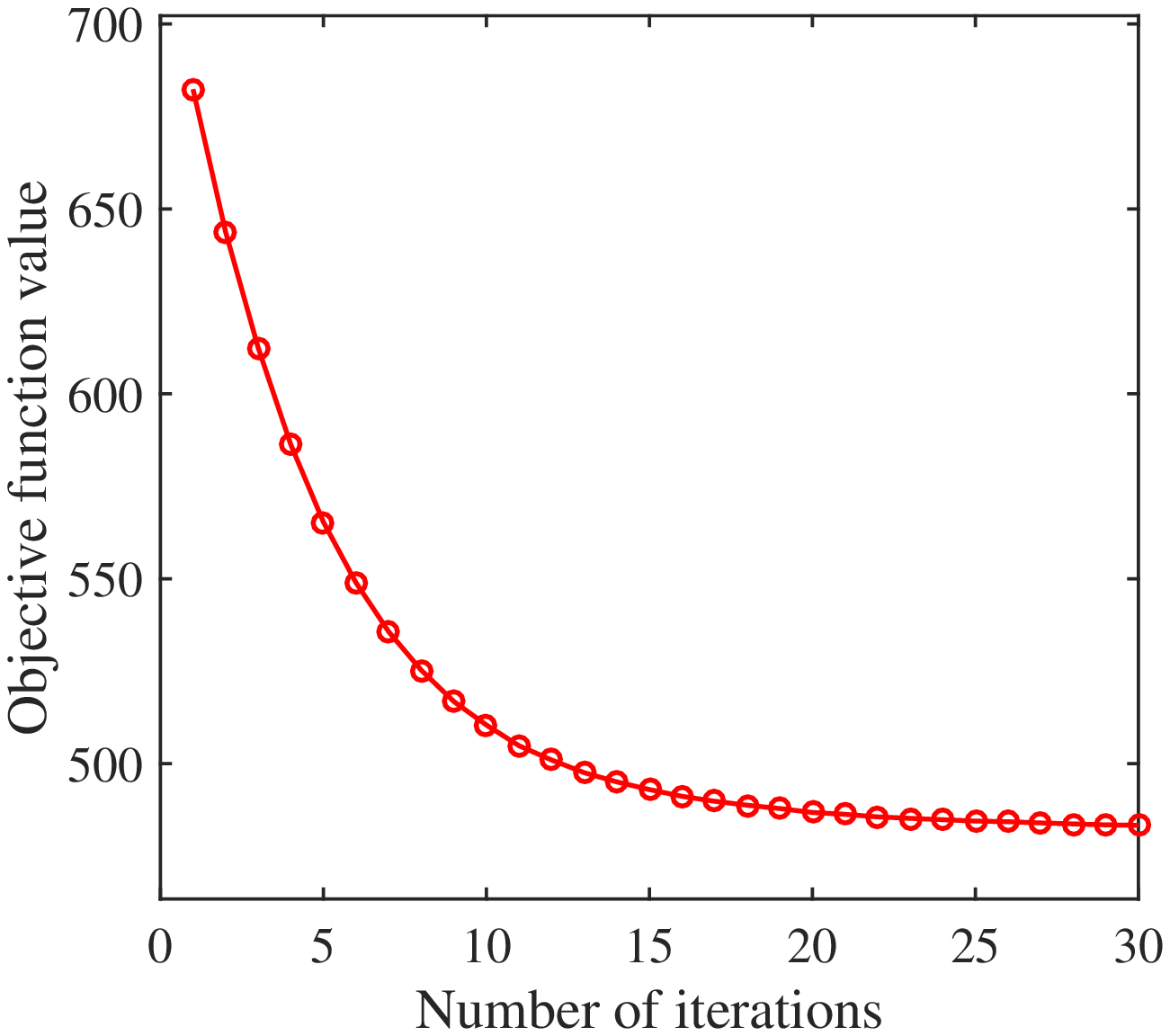}}
  \subfigure[COIL20 (\# 30)]{
  \centering
  \includegraphics[scale=0.22]{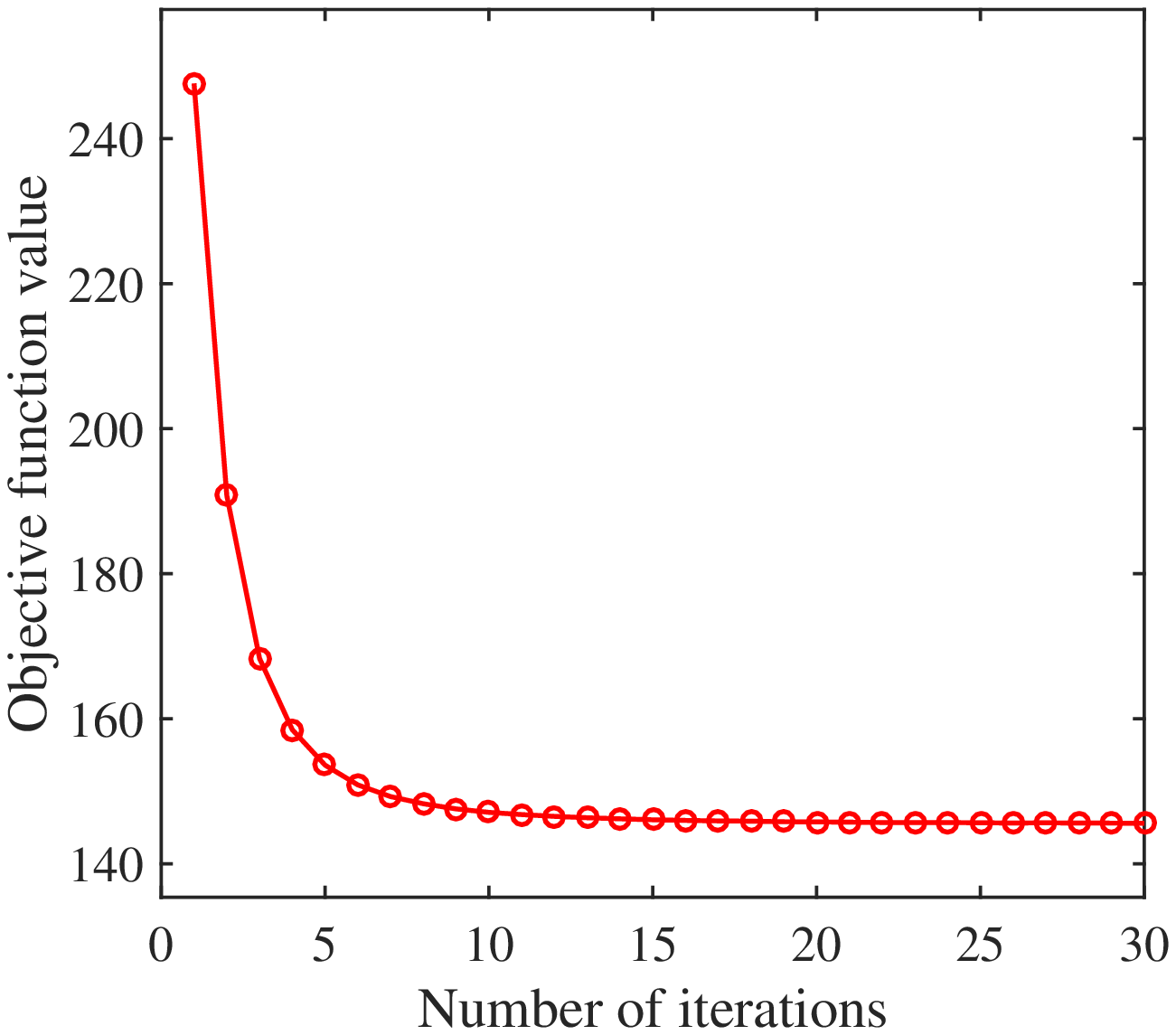}}
  \subfigure[Caltech101 (\# 25)]{
  \centering
  \includegraphics[scale=0.22]{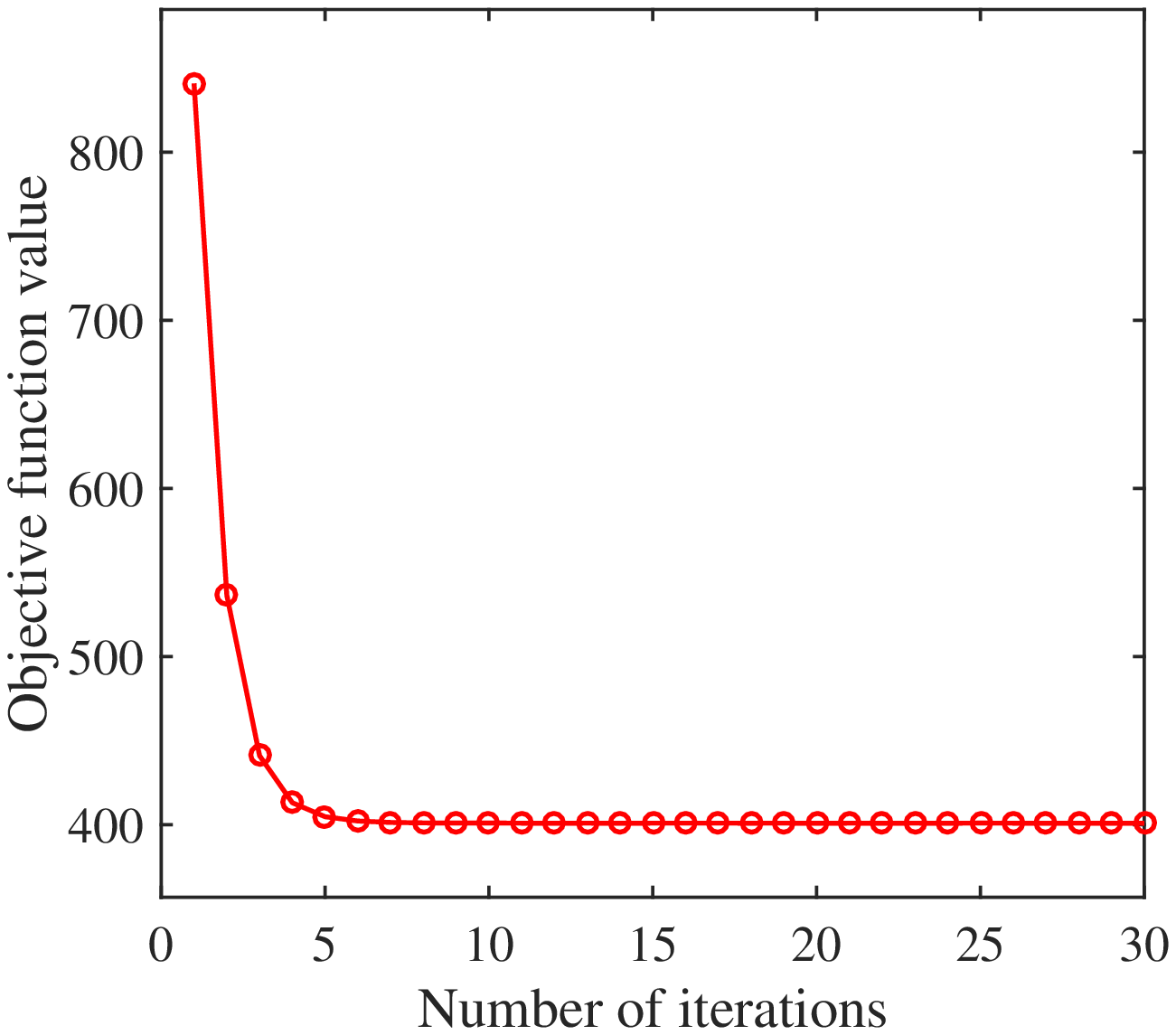}}

  \caption{Convergence curves of the proposed RLAR on ten different databases.}\label{fig:Convergence}
\end{figure*}

\subsection{Analysis of Experimental Results}
Combining theoretical and experimental results, we concentrate on discussing the following insights.
\begin{enumerate}
  \item It can be concluded from the above results in Tables \ref{tab:UCI}-\ref{tab:original:occlusion} that our RLAR surpasses other state-of-the-art approaches and can survive in multiple application scenarios with a relatively efficient and robust classification. And the tuning of the three hyper-parameters that our RLAR handles can be achieved through a simple grid search. The visualization results in Fig.~\ref{fig:Visualization} and the convergence curves in Fig.~\ref{fig:Convergence} experimentally support the aforementioned theoretical assumptions and derivations.
  \item LFDA and NMMP et al. developed the point-to-center loss of LDA as the point-to-point loss, and endowed interlinked samples affinity, breaking through the limit of Gaussian distribution, and the above experimental results indeed confirmed the effectiveness of this move. Inspired by this, we convert the fully-connected intra-class scatter loss of LDA into the loss of partial connection, and effectively overcome the mean dependence and sensitivity to outliers of LDA and its variants by formalizing the affinity between neighbor samples with the non-squared $L_2$ norm.
  \item Generally speaking, our approach performs more stable and efficient than these state-of-the-art discriminant FE methods in classification efficiency, which indicates that the proposed new discriminant criteria are indeed effective and have the stronger ability to extract the discriminant information. Besides, compared with the manifold-inspired method MPDA, the underlying structure of the data preserved more effectively by forcibly separating the submanifolds on which different classes of samples are attached while revealing the overall manifold structure of the data.
  \item The above experimental results demonstrate that multi-classification regression methods RR and ReLSR are relatively resultful for multi-scene recognition, thus suggesting that they are indeed valid in separating samples belonging to different classes in the latent subspace. Moreover, ReLSR is more versatile than RR in most comparisons, which indicates that learning regression targets with large margins of different classes is beneficial to classification performance.
  \item Differ from the conventional methods of depicting the loss function with the squared $L_2$ norm, all the modules in our model are directly measured with the $L_{2,1}$ norm of matrix, which not only enables our model to have the ability of anti-noise, but also can realize the joint process of subspace learning and feature selection. Specially designed sensitivity experiments to outliers also verify the robustness of this strategy.
\end{enumerate}

\section{Conclusion}\label{conclusion}
In this paper, we succeed in achieving a more robust and discriminative low-dimensional representation of data, which is suitable for labeled data classification in multiple scenarios. The proposed model can both adaptively reveal the local structure of the data manifold and flexibly learn the margin representation. All the modules in our model are measured by $L_{2,1}$ norms, thus achieving the joint robust subspace learning and feature selection. Furthermore, we derive an alternate iterative optimization algorithm which is theoretically proved to converge. Extensive experiments conducted on several UCI and other real-world databases have demonstrated the robustness to outliers and classification efficiency of the proposed method. Although we employ retargeted regression here as a way to induce the margin representation, there are certain limitations on the low-dimensional representations, which is where our future work needs to focus.


%

\appendices
\section{Proof of Eq.~(\ref{eq:LDA:intraclass})}\label{LDA:intra-class}
\begin{equation*}
  \begin{split}
  &\sum_{i=1}^c\sum_{j=1}^{n_i}\left\|\bm X^i_j-\bm M_i\right\|_2^2\\
  =&\sum_{i=1}^c\sum_{j=1}^{n_i}(\bm X^i_j-\frac{1}{n_i}\sum_{k=1}^{n_i}\bm X^i_k)^T(\bm X^i_j-\frac{1}{n_i}\sum_{k=1}^{n_i}\bm X^i_k)\\
  =&\sum_{i=1}^c\!\sum_{j=1}^{n_i}({\bm X^i_j}^T\!\bm X^i_j\!-\!\frac{2}{n_i}{\bm X^i_j}^T\!\sum_{k=1}^{n_i}\!\bm X^i_k\!+\!\frac{1}{n_i^2}\!\sum_{j,k=1}^{n_i}\!{\bm X^i_j}^T\!\bm X^i_k)\\
  =&\sum_{i=1}^c(\sum_{j=1}^{n_i}{\bm X^i_j}^T\bm X^i_j-\frac{1}{n_i}\sum_{j,k=1}^{n_i}{\bm X^i_j}^T\bm X^i_k)\\
    =&\sum_{i=1}^c\frac{1}{2n_i}\sum_{j,k=1}^{n_i}({\bm X^i_j}^T\bm X^i_j-2{\bm X^i_j}^T\bm X^i_k+{\bm X^i_k}^T\bm X^i_k)\\
  =&\sum_{i=1}^c\frac{1}{2n_i}\sum_{j,k=1}^{n_i}(\bm X^i_j-\bm X^i_k)^T(\bm X^i_j-\bm X^i_k)\\
 =&\sum_{i=1}^c\sum_{j,k=1}^{n_i}\frac{1}{2n_i}\|\bm X^i_j-\bm X^i_k\|_2^2\\
  \end{split}
\end{equation*}

\section*{Acknowledgment}

The authors would like to thank...

\ifCLASSOPTIONcaptionsoff
  \newpage
\fi

\end{document}